\documentclass[lettersize,journal]{IEEEtran}
\usepackage{amsmath,amsfonts,amssymb}
\usepackage{algorithmic}
\usepackage{algorithm}
\usepackage{array}
\usepackage[caption=false,font=normalsize,labelfont=sf,textfont=sf]{subfig}
\usepackage{textcomp}
\usepackage{stfloats}
\usepackage{url}
\usepackage{verbatim}
\usepackage{graphicx}
\def\BibTeX{{\rm B\kern-.05em{\sc i\kern-.025em b}\kern-.08em
    T\kern-.1667em\lower.7ex\hbox{E}\kern-.125emX}}
\usepackage{balance}
\usepackage{cite}
\usepackage{amsthm}

\usepackage{booktabs}
\usepackage{multirow}
\usepackage{newtxtext}
\usepackage{newtxmath}
\usepackage{placeins}

\newtheorem{theorem}{Theorem}
\newtheorem{lemma}{Lemma}
\newtheorem{definition}{Definition}
\newtheorem{assumption}{Assumption}

\newtheorem{corollary}{Corollary}
\newtheorem{remark}{Remark}

\begin{document}

\title{Beyond Semantic Priors: Mitigating Optimization Collapse for Generalizable Visual Forensics}

\author{Jipeng~Liu,
Haichao~Shi,
Siyu~Xing, Rong~Yin, 
and~Xiao-Yu~Zhang,
\thanks{Jipeng Liu, Haichao Shi, Siyu Xing, and Xiao-Yu Zhang are with the Institute of Information Engineering, Chinese Academy of Sciences, Beijing 100193, China (e-mail: liujipeng@iie.ac.cn; shihaichao@iie.ac.cn; xingsiyu@iie.ac.cn; zhangxiaoyu@iie.ac.cn). Rong Yin is with the School of Cyber Science and Technology, Beihang University, Beijing 100191, China (e-mail: yinrong@buaa.edu.cn). \textit{(Corresponding author: Xiao-Yu Zhang.)}}
}

\markboth{Preprint. Under Review.}
{J. Liu \MakeLowercase{\textit{et al.}}: Mitigating Optimization Collapse for Generalizable Visual Forensics}

\maketitle

\begin{abstract}
While Vision-Language Models (VLMs) like CLIP have emerged as a dominant paradigm for generalizable deepfake detection, a fundamental representational disconnect remains: their semantic-centric pre-training is ill-suited for capturing the non-semantic artifacts inherent to hyper-realistic synthesis.
In this work, we identify a failure mode termed \textit{Optimization Collapse}, where detectors trained with Sharpness-Aware Minimization (SAM) degenerate to random guessing on non-semantic forgeries once the perturbation radius $\rho$ exceeds a narrow threshold.
To theoretically formalize this collapse, we propose the \textit{Critical Optimization Radius (COR)} to quantify the geometric stability of the optimization landscape, and leverage the Gradient Signal-to-Noise Ratio (GSNR) to measure generalization potential. 
We establish a theorem proving that COR increases monotonically with GSNR, thereby revealing that the geometric instability of SAM optimization originates from degraded intrinsic generalization potential.
This result identifies the layer-wise attenuation of GSNR as the root cause of Optimization Collapse in detecting non-semantic forgeries.
Although naively reducing $\rho$ yields stable convergence under SAM, it merely treats the symptom without mitigating the intrinsic generalization degradation, which necessitates the enhancement of gradient fidelity.
Building on this insight, we propose the \textit{Contrastive Regional Injection Transformer (CoRIT)}, which integrates a computationally efficient Contrastive Gradient Proxy (CGP) with three training-free strategies: Region Refinement Mask to suppress CGP variance, Regional Signal Injection to preserve CGP magnitude, and Hierarchical Representation Integration to attain more generalizable representations.
Extensive experiments demonstrate that CoRIT mitigates optimization collapse and achieves state-of-the-art generalization across cross-domain and universal forgery benchmarks.
\end{abstract}


\section{Introduction}
The adversarial dynamics between neural synthesis and forensic detection have intensified into a perpetual arms race. As generative models evolve from early GANs~\cite{DBLP:conf/nips/GoodfellowPMXWOCB14, DBLP:conf/cvpr/ChoiCKH0C18} to recent diffusion Models~\cite{DBLP:conf/nips/HoJA20,DBLP:conf/icml/RameshPGGVRCS21} and sophisticated face-swapping protocols~\cite{DBLP:conf/eccv/XuZHTZTWW022}, the visual fidelity of synthesized content has improved dramatically. However, distinct generation algorithms inevitably leave unique, often imperceptible, statistical fingerprints~\cite{DBLP:journals/pami/AsnaniYHL23}. This creates a generalization bottleneck for detectors: while they excel at identifying seen forgery patterns, their performance often drops significantly when encountering unseen generative techniques.

Historically, forgery detection began with handcrafted representations, targeting blending boundary anomalies~\cite{DBLP:conf/cvpr/LiBZYCWG20,DBLP:journals/pami/NirkinWKH22,DBLP:conf/cvpr/ShioharaY22} or up-sampling artifacts~\cite{DBLP:conf/cvpr/TanLZWGLW24,DBLP:journals/corr/abs-2504-04827}. 
However, their effectiveness is bound to specific priors, failing when these priors are violated.
To overcome this limitation, researchers shifted towards data-driven deep representations, with Convolutional Neural Networks such as XceptionNet~\cite{DBLP:conf/cvpr/Chollet17, DBLP:journals/tip/ChengZZYLWL25,liu2025knowledge,DBLP:conf/iccv/0002ZFW23,DBLP:conf/eccv/QianYSCS20,DBLP:journals/pami/ZhangZZZGZZY25} and EfficientNet~\cite{DBLP:conf/icml/TanL19,DBLP:conf/cvpr/0002LLLW24} becoming the backbone of detection frameworks. 
Although these models excel at extracting high-frequency anomalies~\cite{DBLP:conf/eccv/QianYSCS20,DBLP:conf/aaai/Tan0WGLW24} and model fingerprints~\cite{DBLP:conf/cvpr/GuillaroCSDV23,DBLP:journals/pami/AsnaniYHL23}, their inherent inductive bias towards local textures leads to overfitting local-specific artifacts, resulting in limited generalization and universality.
Recently, a paradigm shift has emerged towards leveraging the universal capabilities of Vision-Language Models for forgery detection, predominantly adopting CLIP~\cite{DBLP:conf/icml/RadfordKHRGASAM21,shao2024detecting,DBLP:conf/cvpr/GuoSZLL25} as the foundational architecture.
Ojha et al.~\cite{DBLP:conf/cvpr/OjhaLL23} pioneered this direction by revealing that the frozen CLIP representation inherently encapsulates generalizable artifacts, enabling the detection of fully synthetic images via simple linear probing.
However, the efficacy of such frozen features remains limited when applied to specific facial manipulation tasks. 
To address this, subsequent research has explored Parameter-Efficient Fine-Tuning (PEFT) strategies to adapt CLIP~\cite{DBLP:conf/cvpr/CuiLLZD25,DBLP:conf/icml/0002WJZLCYDW025,shao2024detecting,DBLP:conf/cvpr/GuoSZLL25}.
For instance, Cui et al.~\cite{DBLP:conf/cvpr/CuiLLZD25} introduced lightweight adapters and auxiliary convolutional branches. 
Explicitly reintroducing convolutional priors, however, risks re-inheriting the texture-specific overfitting tendencies characteristic of traditional CNNs.
From an optimization perspective, standard Stochastic Gradient Descent (SGD) inherently tends to overfit the local textures of these non-semantic forgeries. Sharpness-Aware Minimization (SAM)~\cite{DBLP:conf/iclr/ForetKMN21} has emerged to improve detector generalization by flattening the loss landscape, as utilized by Ma et al.~\cite{DBLP:journals/corr/abs-2504-04827}.

Despite the growing consensus on utilizing CLIP visual encoders for forensics, a fundamental representational disconnect remains unaddressed. CLIP is pre-trained specifically for visual-language alignment, extracting features that prioritize semantic content. Conversely, the consensus posits that deepfake artifacts reside in the non-semantic, high-frequency domain~\cite{DBLP:conf/eccv/QianYSCS20, DBLP:conf/aaai/Tan0WGLW24}, manifesting as subtle "model fingerprints"~\cite{DBLP:conf/cvpr/GuillaroCSDV23,DBLP:journals/pami/AsnaniYHL23} hidden beneath overt semantic anomalies. This latent characteristic is paramount, as the inevitable evolution toward high-fidelity synthesis threatens to eliminate semantic discrepancies entirely. This raises a fundamental question regarding the safety and reliability of CLIP-based forensics detectors: 

\textit{Can the semantic-centric feature space of CLIP effectively encapsulate the generalizable non-semantic artifacts inherent to increasingly hyper-realistic forgery techniques?}

\begin{figure}[t]
    \centering
    \includegraphics[width=\linewidth]{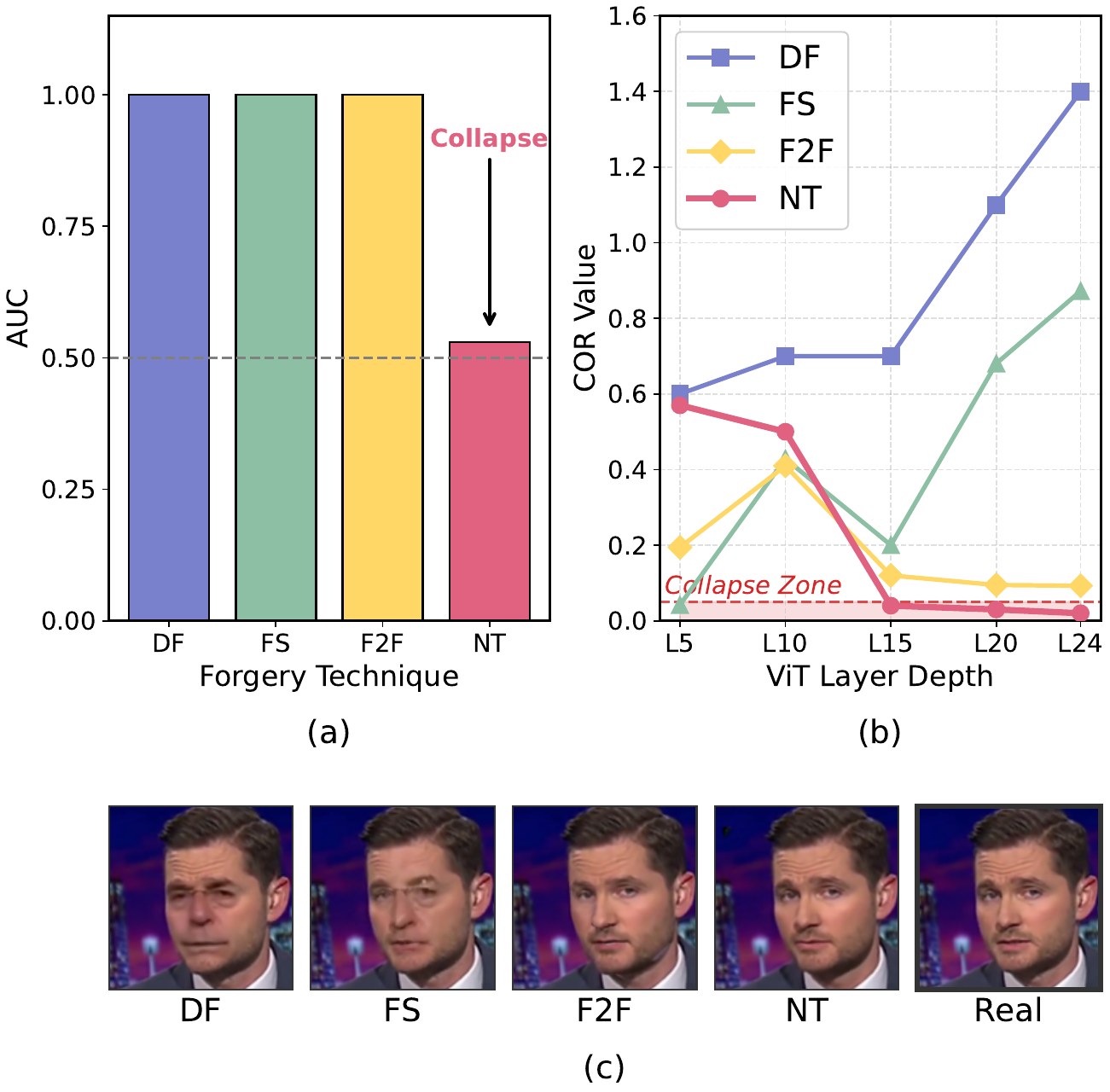}
    \vspace{-10pt} 
    \caption{Empirical observation of Optimization Collapse. 
        (a) Detection performance at the standard SAM radius ($\rho=0.05$). 
        (b) Layer-wise evolution of the Critical Optimization Radius (COR). The red region denotes the "Collapse Zone" (COR $< 0.05$). 
        (c) Representative samples from the evaluated datasets.}
    \label{fig:collapse_analysis}
\end{figure}

To investigate this concern, we re-examined diverse forgery paradigms and identified a notable optimization failure termed \textit{Optimization Collapse}. 
As illustrated in Fig.~\ref{fig:collapse_analysis}(a), training a lightweight classifier on frozen CLIP-ViT-L/14 features exposes this collapse: SAM-optimized detectors drop to random guessing on NeuralTextures (NT) once the perturbation radius $\rho$ exceeds a narrow threshold.
Under the officially recommended SAM hyperparameter ($\rho=0.05$, detailed in Section~\ref{sec:preliminary}),
the model achieved perfect convergence on the Face2Face (F2F), DeepFakes (DF), and FaceSwap (FS) datasets, yet collapsed to random guessing on NT.
Effective learning on NT was only restored when $\rho$ was reduced to 0.01.
We define this breakdown threshold as the \textit{Critical Optimization Radius (COR)}. 
Furthermore, as depicted in Fig.~\ref{fig:collapse_analysis}(b), the layer-wise evolution of COR reveals a distinct bifurcation. 
For coarse forgeries (DF, FS), the COR trajectories ascend through the layers, while for subtle manipulations (F2F, NT), they rapidly decay.
At the final transformer layer, the COR for NT drops into a "Collapse Zone" (COR $< 0.05$), providing an explanation for the failure of NT observed in Fig.~\ref{fig:collapse_analysis}(a).
Fig.~\ref{fig:collapse_analysis}(c) corroborates this finding: DF and FS exhibit overt semantic anomalies, whereas NT and F2F rely on non-semantic discrepancies that are imperceptible to the human eye.
This correspondence is expected given the semantic-centric feature space of CLIP-ViT, which systematically filters out non-semantic signals while progressively aggregating semantic features across layers~\cite{DBLP:conf/iclr/ParkK22}.

Intuition illuminates the phenomenon, yet obscures the mechanism. We formulate a rigorous theoretical framework that bridges the geometric stability of SAM and the Gradient Signal-to-Noise Ratio (GSNR)~\cite{DBLP:conf/iclr/LiuBJCW20}, a statistical metric quantifying generalizability. 
By introducing the geometric stability condition where the SAM update gradient maintains a valid descent direction, we derive an analytical expression for the COR.
Furthermore, we prove that COR increases monotonically with the GSNR at the bottleneck step. 
This theorem elucidates the essence of Optimization Collapse: it is not merely a geometric instability induced by gradient perturbations, but a direct manifestation of degraded GSNR arising from insufficient intrinsic generalizability. COR and GSNR serve as universal quantitative metrics for assessing the intrinsic generalization potential of forensic models.

This theoretically substantiates our initial concern with the empirical data in Fig.~\ref{fig:collapse_analysis}(b): \textit{the semantic-centric inductive bias of deep CLIP layers inevitably creates a generalization blind spot for non-semantic, high-fidelity deepfakes.} 
Optimization Collapse is driven by this inherent generalization degradation, which amplifies the model's vulnerability to minor gradient perturbations $\rho$.
Although naively reducing $\rho$ yields stable convergence under SAM, it merely treats the symptom without rectifying this intrinsic representational deficit, which necessitates the enhancement of gradient fidelity.

To transcend the non-semantic generalization barrier of semantic-centric models, guided by our theoretical framework, we introduce the \textbf{Co}ntrastive \textbf{R}egional \textbf{I}njection \textbf{T}ransformer (CoRIT).
For computational efficiency, we derive a Contrastive Gradient Proxy (CGP) to approximate the gradient.
Armed with this proxy, our approach incorporates three training-free strategies:
(i) A Region Refinement Mask is introduced to suppress CGP variance by identifying spatially coherent tokens with high CGP response; 
(ii) A Regional Signal Injection mechanism iteratively injects the refined high-GSNR information into additional Region Tokens to counteract the layer-wise attenuation of non-semantic signals; 
and (iii) A Hierarchical Representation Integration fuses intermediate signals rich in non-semantic forensic artifacts with deep semantic abstractions to yield generalizable multi-scale representations.
By enhancing gradient signal fidelity, this framework elevates the COR out of the collapse regime ($\text{COR} < 0.05$), thereby effectively boosting the model's cross-domain generalizability.

Our main contributions are summarized as follows:
\begin{itemize}
    \item We identify the Optimization Collapse phenomenon in CLIP-based detectors: SAM-optimized detectors drop to random guessing on high-fidelity forgeries once the perturbation radius $\rho$ exceeds a narrow threshold.  
    We theoretically formalize this by introducing the Critical Optimization Radius (COR) and the Gradient Signal-to-Noise Ratio (GSNR), which jointly expose an endogenous mismatch between generalization potential and optimization geometry.
    
    \item We establish a theorem proving that COR is monotonically increasing in GSNR, revealing that a vanishing COR is the inevitable consequence of degraded GSNR. This insight bridges the gap between geometric stability and generalizability, elucidating that Optimization Collapse is an optimization instability arising from insufficient intrinsic generalizability in semantic-centric backbones.
    
    \item Guided by our theoretical framework, we propose CoRIT, a backbone-frozen architecture designed to maximize COR, which integrates a computationally efficient Contrastive Gradient Proxy with three training-free strategic components: Region Refinement Mask, Regional Signal Injection, and Hierarchical Representation Integration. This framework enhances gradient signal fidelity to mitigate optimization collapse, actively counteracting the layer-wise attenuation of non-semantic forensic signals.
    
    \item Evaluations across comprehensive cross-domain and universal forgery benchmarks show that CoRIT outperforms existing state-of-the-art approaches. These results validate our theory-driven design in transcending the generalization barrier for high-fidelity deepfake detection.
\end{itemize}
\section{Related Work}

\subsection{Forensic Representations}
Effective and generalizable forgery representations are the foundation of deepfake detection. The evolution of forensic representations can be categorized into three distinct paradigms: Handcrafted Forensics, Deep Forensics, and the emerging Universal Forensics enabled by vision-language models.

\subsubsection{Handcrafted Forensics}
Early deepfake detection methods relied on handcrafted representations derived from specific physical priors, such as biological photoplethysmography (PPG) signals~\cite{ciftci2020fakecatcher} or the disentanglement of 3D shape, lighting, and texture~\cite{DBLP:journals/pami/ZhuFZZZLL23}. Targeting blending artifacts, a prominent line of research exploited the discrepancies between the manipulated facial foreground and the background context, either by detecting blending boundaries~\cite{DBLP:conf/cvpr/LiBZYCWG20} or by analyzing face-context identity inconsistencies~\cite{DBLP:journals/pami/NirkinWKH22}. Within this scope, Shiohara et al.~\cite{DBLP:conf/cvpr/ShioharaY22} introduced Self-Blended Images (SBI) to synthesize pseudo-forgeries entirely from real images, guiding models to learn generalizable blending boundaries.
Recently, handcrafted representations have evolved to capture structural artifacts in advanced full-image synthesis models, utilizing gradients~\cite{tan2023learning}, neighboring pixel relationships~\cite{DBLP:conf/cvpr/TanLZWGLW24} to expose up-sampling anomalies, or integrating structural and blending anomalies~\cite{DBLP:journals/corr/abs-2504-04827}.
Despite their effectiveness on specific anomalies, the strict reliance on predefined priors restricts their generalization when confronting unseen manipulations, necessitating the transition to data-driven learning paradigms.

\subsubsection{Deep Forensics}
Deep representation learning has subsequently become the mainstream paradigm, utilizing fully fine-tuned (FFT) backbones to extract data-driven forgery features. 
Representative architectures include XceptionNet~\cite{DBLP:conf/cvpr/Chollet17, DBLP:journals/tip/ChengZZYLWL25,liu2025knowledge,DBLP:conf/iccv/0002ZFW23,DBLP:conf/eccv/QianYSCS20,DBLP:journals/pami/ZhangZZZGZZY25}, recognized as a strong baseline on FaceForensics++~\cite{DBLP:conf/iccv/RosslerCVRTN19}, alongside EfficientNet~\cite{DBLP:conf/icml/TanL19,DBLP:conf/cvpr/0002LLLW24} and ResNet~\cite{DBLP:conf/cvpr/HeZRS16,DBLP:journals/pami/QuLLWC26,DBLP:journals/pami/ZhangZZZGZZY25,DBLP:conf/cvpr/KashianiTA25}.
Since forgery artifacts often manifest in the high-frequency domain, researchers have incorporated frequency information through learnable filters~\cite{DBLP:conf/eccv/QianYSCS20} or adaptive transforms~\cite{DBLP:conf/aaai/Tan0WGLW24} to expose spectral anomalies.
Other approaches exploit the unique fingerprints inherent to generative models. Guillaro et al.~\cite{DBLP:conf/cvpr/GuillaroCSDV23} automatically extract these traces for detection, while Asnani et al.~\cite{DBLP:journals/pami/AsnaniYHL23} leverage them to reverse engineer the source generator for trustworthy attribution.
Despite these advancements, relying solely on standard visual backbones introduces an inductive bias towards local textures, causing models to overfit dataset-specific artifacts. 
To enhance generalization, researchers have incorporated advanced strategies like disentangling common features~\cite{DBLP:conf/iccv/0002ZFW23} and latent space augmentation~\cite{DBLP:conf/cvpr/0002LLLW24} to transcend forgery specificity. 
Qiao et al.~\cite{DBLP:journals/pami/QiaoXCRL24} explored unsupervised learning paradigms to eliminate the dependency on annotated forgery labels. 
Identity analysis also proves effective, detecting manipulations by retrieving source identities~\cite{DBLP:journals/pami/ZhangZZZGZZY25} or verifying identity consistency~\cite{DBLP:journals/pami/QuLLWC26}.
Parallel efforts explicitly mitigate domain-specific biases through data debiasing~\cite{DBLP:journals/tip/ChengZZYLWL25}, frequency debiasing~\cite{DBLP:conf/cvpr/KashianiTA25}, spatial debiasing~\cite{DBLP:conf/aaai/FuYYCL25}, and automated debiasing~\cite{liu2025knowledge}.
The confinement of task-specific supervision to closed-set distributions necessitates a shift towards foundation models that inherently offer universal representations.

\subsubsection{Universal Forensic}
In the era of Vision-Language Models (VLMs), Ojha et al.~\cite{DBLP:conf/cvpr/OjhaLL23} revealed that the universal representations of foundation models inherently capture generalizable forgery traces, demonstrating that linear probing on frozen CLIP features enables universal detection of entire synthetic images.
But its efficacy on facial manipulation remains limited. 
To bridge this gap, recent CLIP adaptation strategies employ Parameter-Efficient Fine-Tuning (PEFT) paradigms, utilizing forensics-oriented adapters~\cite{DBLP:conf/cvpr/CuiLLZD25} or advanced LoRA schemes leveraging orthogonal subspace decomposition~\cite{DBLP:conf/icml/0002WJZLCYDW025}.
Emerging research leverages CLIP as a foundational encoder for multi-modal forgery detection~\cite{shao2024detecting,DBLP:conf/cvpr/GuoSZLL25}.
With CLIP established as the representative backbone of this paradigm shift, however, a profound theoretical conflict is overlooked: while CLIP is pre-trained for semantic alignment, widely recognized forensic clues, such as up-sampling artifacts, high-frequency anomalies, and model fingerprints, are inherently non-semantic.
This raises a fundamental question: Can semantic-centric backbones effectively characterize non-semantic forensic traces?
Addressing this, we provide a theoretical analysis and design a method to explicitly enhance non-semantic forensic representations.

\subsection{Sharpness-Aware Minimization}
Sharpness-Aware Minimization (SAM)~\cite{DBLP:conf/iclr/ForetKMN21} improves generalization by minimizing the worst-case loss within a Euclidean neighborhood. Expanding on this geometric foundation, Fisher SAM~\cite{DBLP:conf/icml/Kim0HH22} models the parameter space as inherently non-Euclidean, proposing an ellipsoid neighborhood aligned with the Fisher Information Matrix. Tilted SAM~\cite{DBLP:conf/icml/00050B25} addresses the non-smoothness of the minimax objective via exponential tilting to explicitly favor flatness.
Beyond static geometric analysis, recent studies have explored dynamic optimization strategies to enhance SAM's efficiency and effectiveness. Sharpness-Aware Lookahead (SALA)~\cite{DBLP:journals/pami/TanZLG24} combines the fast convergence of Lookahead with the generalization benefits of SAM by seeking flat minima in the later stages of training. 
To mitigate computational overhead, Sparse SAM (SSAM)~\cite{DBLP:journals/pami/MiSRZXSLJT25} applies sparse perturbations, establishing that indiscriminate perturbation remains suboptimal.

While these variants refine the framework via static geometric analysis or dynamic optimization strategies, Unified SAM (USAM)~\cite{DBLP:conf/iclr/OikonomouL25} unifies update rules under a relaxed "Expected Residual" condition to model dynamic convergence. 
USAM relies on global relaxed scaling that provides uniform bounds; however, it overlooks local landscape variations, making it insufficient to explain the Optimization Collapse phenomenon we observe.
We bypass these global limitations by employing \textit{Local Lipschitz Smoothness} to achieve a fine-grained characterization of geometric stability in SAM optimization. Furthermore, we establish a theoretical connection between the geometric Critical Optimization Radius (COR) and the Gradient Signal-to-Noise Ratio (GSNR) to quantify generalizability.
\begin{figure*}[t]
    \centering
    \includegraphics[width=0.99\textwidth]{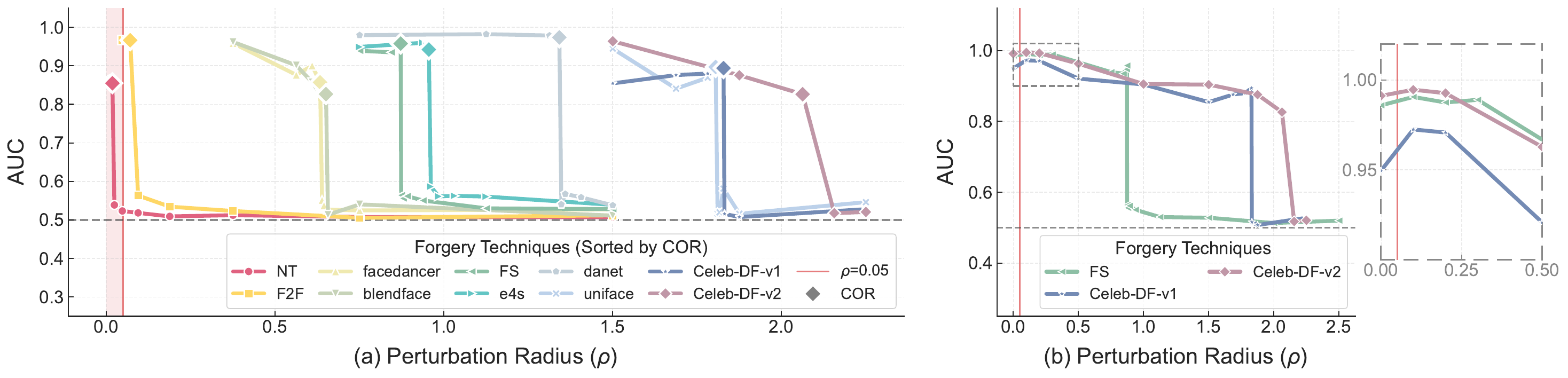}
    
    \caption{AUC curves of models trained with Sharpness-Aware Minimization (SAM) across varying perturbation radii $\rho$. (a) Empirical analysis of Optimization Collapse. The Critical Optimization Radius (COR) (diamond markers) across various datasets. The red region highlights optimization failure at the standard radius ($\rho=0.05$). (b) AUC trends across the entire span of $\rho$. The gray dashed box indicates the localized region magnified on the right, highlighting AUC variations around the recommended radius ($\rho=0.05$).}
    \label{fig:oc_visual} 
\end{figure*}

\section{Theoretical Analysis of Optimization Collapse: From Geometric Stability to Generalization}
\label{sec:optimization_collapse}

The empirical bifurcation observed in Fig.~\ref{fig:collapse_analysis} exposes an optimization failure phenomenon termed Optimization Collapse. 
To quantify this phenomenon, we first propose the Critical Optimization Radius (COR) and derive its analytical expression by imposing the condition of critical geometric stability. Subsequently, we establish a theoretical link between COR and the Gradient Signal-to-Noise Ratio (GSNR)~\cite{DBLP:conf/iclr/LiuBJCW20}, a generalization metric quantifying the quality of the gradient signal. This connection demonstrates that a vanishing COR stems directly from a degraded GSNR, identifying insufficient generalization as the root cause of the collapse.
This theoretical insight guides the design of our proposed framework.

\subsection{The Phenomenon of Optimization Collapse}
\label{sec:phenomenon_collapse}

Optimization Collapse occurs when the classifier's predictive capability drops to random guessing once the SAM perturbation radius $\rho$ exceeds a narrow threshold. 
To investigate this, we systematically profiled AUC trajectories using linear probing on the frozen CLIP-ViT-L/14 backbone across ten representative face forgery paradigms. 
As Fig.~\ref{fig:oc_visual}(a) illustrates, a consistent pattern emerges: when $\rho$ exceeds its respective threshold, the AUC falls to approximately 0.5. 
We define this boundary as the Critical Optimization Radius (COR). 
A sufficiently large COR ensures stable convergence under SAM; however, a vanishing COR, as observed in NT and F2F, invalidates the SAM update.

Can simply reducing $\rho$ resolve this collapse? We emphasize that SAM is employed here not as an indispensable optimizer, but as a controlled diagnostic probe exposing the intrinsic representational limitations of the backbone. As illustrated in Fig.~\ref{fig:oc_visual}(b), increasing $\rho$ from 0 to 0.1 yields consistent generalization gains, confirming the well-established benefits of landscape flattening. When $\rho = 0$, SAM degenerates to standard SGD, which converges stably but completely forfeits these generalization advantages. Naively reducing $\rho$ circumvents the geometric instability yet simultaneously abandons the generalization potential, leaving the underlying representational deficit entirely unresolved. Our subsequent theoretical analysis formalizes this tension by linking the stability boundary (COR) to gradient signal quality (GSNR), ultimately reframing Optimization Collapse from a mere optimization failure into a profound representational diagnosis of intrinsic generalization capability.

\subsection{Preliminaries and Definitions}
\label{sec:preliminary}

We adopt standard notation to characterize the optimization dynamics. Let $\mathbf{w}_t$ denote the model parameters at step $t$ approaching a local minimum $\mathbf{w}^*$. For the global objective function $\mathcal{L}(\mathbf{w})$, its value, gradient, and Hessian matrix at step $t$ are denoted as $\mathcal{L}_t \triangleq \mathcal{L}(\mathbf{w}_t)$, $\nabla \mathcal{L}_t \triangleq \nabla \mathcal{L}(\mathbf{w}_t)$, and $\mathbf{H}_t \triangleq \nabla^2 \mathcal{L}(\mathbf{w}_t)$, respectively. Let $\mathcal{B}_t$ denote the mini-batch sampled at step $t$. The stochastic gradient is defined as $\mathbf{g}_t \triangleq \nabla \mathcal{L}_{\mathcal{B}_t}(\mathbf{w}_t)$, which serves as an unbiased estimator of the global gradient, i.e., $\mathbb{E}_{\mathcal{B}_t}[\mathbf{g}_t] = \nabla \mathcal{L}_t$. Let $\tilde{\mathbf{g}}_t$ represent the SAM-perturbed gradient.
To analyze the stability limits, we define $\rho$ as the perturbation radius, $\rho_{\text{critical}}$ as the Critical Optimization Radius (COR), and $t^*$ as the Bottleneck Step. We quantify the local geometry via the local smoothness constant $L_t$ and the trace-based stable rank of the Hessian $\kappa_s(\mathbf{H}_t) \triangleq \frac{\operatorname{Tr}(\mathbf{H}_t)}{\lambda_{\max}(\mathbf{H}_t)}$, which measures the effective dimensionality of the Hessian's eigenspectrum. Statistical properties are measured by the gradient variance $\sigma^2$ and Gradient Signal-to-Noise Ratio ($\text{GSNR}_t$). 
$\mathbb{E}[\cdot]$, $\text{Var}[\cdot]$, and $\operatorname{Cov}[\cdot]$ denote the expectation, variance, and covariance operators; $\operatorname{Tr}(\cdot)$ and $\lambda_{\max}(\cdot)$ denote the trace and the maximum eigenvalue of a square matrix, respectively; and $\|\cdot\|$ represents the Euclidean norm.

\subsubsection{Sharpness-Aware Minimization (SAM)} 
Standard SAM seeks parameters lying within a flat loss landscape to enhance generalization by solving the following minimax objective~\cite{DBLP:conf/iclr/ForetKMN21}:
\begin{equation}
    \min_{\mathbf{w}} \max_{\|\boldsymbol{\epsilon}\|_2 \le \rho} \mathcal{L}(\mathbf{w} + \boldsymbol{\epsilon}),
    \label{eq:sam_definition}
\end{equation}
where $\rho \ge 0$ represents a pre-defined perturbation radius.
Specifically, at step $t$, SAM computes the perturbed gradient $\tilde{\mathbf{g}}_t = \nabla \mathcal{L}(\mathbf{w}_t + \boldsymbol{\epsilon}_t)$ via an optimal adversarial perturbation vector $\boldsymbol{\epsilon}_t$. 
As derived by Foret et al.~\cite{DBLP:conf/iclr/ForetKMN21} via a first-order Taylor approximation, this optimal adversarial perturbation aligns with the normalized gradient direction: 
\begin{equation}
    \boldsymbol{\epsilon}_t \approx \rho \mathbf{g}_t / \|\mathbf{g}_t\|_2.
    \label{eq:sam_norm}
\end{equation} 
By penalizing sharp minima, as illustrated in Fig.~\ref{fig:oc_visual}(b), SAM yields consistent generalization gains at a recommended radius ($\rho \approx 0.05$). 
As $\rho$ further increases, the generalization performance gradually degrades, until exceeding the COR triggers Optimization Collapse.

\subsubsection{Geometric Stability Condition}
\label{subsubsec:geometric_stable_condition}
We observe that SAM obtains the perturbed gradient $\tilde{\mathbf{g}}_t$ by injecting an adversarial perturbation $\boldsymbol{\epsilon}_t$ upon the stochastic gradient $\mathbf{g}_t$.
For effective optimization, the intensity of this adversarial perturbation must not overwhelm the original gradient signal. 
We quantify this impact via the expected inner product between the gradient $\nabla \mathcal{L}_t$ and the perturbed update $\tilde{\mathbf{g}}_t$. 
The inner product captures both directional similarity and magnitude similarity, where an excessive perturbation radius $\rho$ renders the descent ineffective by disrupting this similarity.
To guarantee stable descent, we formulate a geometric stability constraint: the perturbed update $\tilde{\mathbf{g}}_t$ must, in expectation, maintain a positive projection onto the true descent direction $\nabla \mathcal{L}_t$:
\begin{equation}
\mathbb{E}[\langle \nabla \mathcal{L}_t, \tilde{\mathbf{g}}_t \rangle] > 0.
\label{eq:stability_condition}
\end{equation}
This inequality serves as the theoretical foundation for deriving the COR.

\subsubsection{Gradient Signal-to-Noise Ratio (GSNR)} 
\label{sub:gsnr_definition}
Introduced by Liu et al.~\cite{DBLP:conf/iclr/LiuBJCW20} to quantify generalization capability, GSNR measures the strength of the true gradient signal relative to the stochastic gradient variance.
To align with the scalar perturbation radius $\rho$, we define the global GSNR at step $t$ as:
\begin{equation}
    \label{eq:gsnr_definition}
    \text{GSNR}_t \triangleq \frac{\|\nabla \mathcal{L}_t\|^2}{\operatorname{Tr}(\operatorname{Cov}[\mathbf{g}_t])}.
\end{equation}
A lower GSNR indicates a degraded gradient signal, which is theoretically linked to poor generalization.

\subsubsection{Hessian-Covariance Decomposition}
\label{sec:hessian_identity}
Second-order optimization theory~\cite{DBLP:journals/jmlr/Martens20} establishes a correspondence between the Hessian's geometric curvature and the gradient covariance's statistical properties.
Under the negative log-likelihood (NLL) loss, let $\mathbf{x} \sim q(\mathbf{x})$ denote the input data and $p_{\mathbf{w}_t}(\mathbf{x})$ be the parameterized model distribution. 
The stochastic gradient is formulated as $\mathbf{g}_t = -\nabla_{\mathbf{w}_t} \log p_{\mathbf{w}_t}(\mathbf{x})$. 
The expected Hessian matrix $\mathbf{H}_t$ decomposes as follows:
\begin{equation}
    \mathbf{H}_t = \operatorname{Cov}[\mathbf{g}_t] + \nabla \mathcal{L}_t \nabla \mathcal{L}_t^\top - \mathbf{\Xi}_t,
    \label{eq:hessian_identity}
\end{equation}
where $\operatorname{Cov}[\mathbf{g}_t]$ is the gradient covariance matrix, and $\mathbf{\Xi}_t \triangleq \mathbb{E}_{\mathbf{x} \sim q(\mathbf{x})} \left[ \frac{\nabla_{\mathbf{w}_t}^2 p_{\mathbf{w}_t}(\mathbf{x})}{p_{\mathbf{w}_t}(\mathbf{x})} \right]$ is the residual term for model misspecification. The details are provided in the supplementary.

\subsection{Assumptions}
\label{sec:theoretical_setup}

\subsubsection{Optimization Setting}
We focus our theoretical analysis on the optimization process trained via mini-batch SGD with Cross-Entropy (CE) loss. For classification task, CE loss coincides with the Negative Log-Likelihood (NLL), which validates the Hessian-Covariance Decomposition in Section~\ref{sec:hessian_identity}.

\begin{assumption}[Local Curvature \& Mini-batch Smoothness]
\label{assum:local_curvature}
Assume the loss function $\mathcal{L}(\mathbf{w})$ is twice continuously differentiable within the $\rho$-neighborhood $\mathcal{N}(\mathbf{w}_t, \rho) \triangleq \{ \mathbf{w} \mid \|\mathbf{w} - \mathbf{w}_t\|_2 \le \rho \}$. We define the local smoothness constant $L_t$ via the supremum of the Hessian's spectral norm, which approximates its maximum eigenvalue for small $\rho$:
\begin{equation}
    L_t \triangleq \sup_{\mathbf{u} \in \mathcal{N}(\mathbf{w}_t, \rho)} \|\nabla^2 \mathcal{L}(\mathbf{u})\|_2 \approx \lambda_{\max}(\mathbf{H}_t).
    \label{eq:local_curvature}
\end{equation}
We assume this $L_t$-Lipschitz smoothness holds uniformly across mini-batches. For any perturbation $\|\boldsymbol{\epsilon}\| \le \rho$:
\begin{equation}
    \|\tilde{\mathbf{g}}_t - \mathbf{g}_t\| = \|\nabla \mathcal{L}_{\mathcal{B}_t}(\mathbf{w}_t + \boldsymbol{\epsilon}) - \nabla \mathcal{L}_{\mathcal{B}_t}(\mathbf{w}_t)\| \le L_t \|\boldsymbol{\epsilon}\|.
    \label{eq:local_smoothness}
\end{equation}

\noindent\textbf{Justification.} The approximation in \eqref{eq:local_curvature} and the mini-batch smoothness assumption are standard in SAM theoretical analysis~\cite{DBLP:conf/iclr/ForetKMN21, DBLP:journals/jmlr/BartlettLB23, DBLP:journals/jmlr/LongB24}. 
Although the bound may become loose for large perturbation radii, it remains tight in the \emph{Optimization Collapse} regime where the critical radius is vanishingly small ($\mathrm{COR} < 0.05$). In such narrow neighborhoods, the maximum eigenvalue of the Hessian exhibits negligible variation, ensuring the approximation's validity.
\end{assumption}

\subsection{Derivation of the Critical Optimization Radius}
\label{sec:derivation_cor}

To analytically formulate the COR, we translate the geometric stability condition (Section~\ref{subsubsec:geometric_stable_condition}) into a quantifiable boundary. This yields a step-wise upper bound for the perturbation radius $\rho$ to guarantee descent.

\begin{lemma}[Local Stability Bound]
\label{lem:local_stability}
Under local $L_t$-Lipschitz smoothness (Assumption~\ref{assum:local_curvature}), the geometric stability condition $\mathbb{E}[\langle \nabla \mathcal{L}_t, \tilde{\mathbf{g}}_t \rangle] > 0$ holds if:
\begin{equation}
    \rho < \frac{\|\nabla \mathcal{L}_t\|}{L_t}.
    \label{eq:local_stability}
\end{equation}
\end{lemma}

\begin{proof}
Decomposing the expected inner product in \eqref{eq:stability_condition} yields:

\begin{equation}
\begin{aligned}
    \mathbb{E}[\langle \nabla \mathcal{L}_t, \tilde{\mathbf{g}}_t \rangle] 
    &= \mathbb{E}[\langle \nabla \mathcal{L}_t, \mathbf{g}_t \rangle] + \mathbb{E}[\langle \nabla \mathcal{L}_t, \tilde{\mathbf{g}}_t - \mathbf{g}_t \rangle] \\ 
    &= \|\nabla \mathcal{L}_t\|^2 + \mathbb{E}[\langle \nabla \mathcal{L}_t, \tilde{\mathbf{g}}_t - \mathbf{g}_t \rangle] \quad && \text{\scriptsize{(Unbiasedness)}}\\
    &\ge \|\nabla \mathcal{L}_t\|^2 - \mathbb{E}[\|\nabla \mathcal{L}_t\| \|\tilde{\mathbf{g}}_t - \mathbf{g}_t\|] \quad && \text{\scriptsize{(Cauchy-Schwarz)}} \\
    &\stackrel{\eqref{eq:local_smoothness}}{\ge} \|\nabla \mathcal{L}_t\|^2 - \mathbb{E}[\|\nabla \mathcal{L}_t\| (L_t \|\boldsymbol{\epsilon}_t\|)] \\
    &\stackrel{\eqref{eq:sam_norm}}{\ge} \|\nabla \mathcal{L}_t\|^2 - L_t \rho \|\nabla \mathcal{L}_t\|.
\end{aligned}
\end{equation}
For \eqref{eq:stability_condition} to hold, we require $\|\nabla \mathcal{L}_t\|^2 - L_t \rho \|\nabla \mathcal{L}_t\| > 0$. Rearranging this inequality yields:

\begin{equation}
    \rho < \frac{\|\nabla \mathcal{L}_t\|}{L_t}.
\end{equation}

\end{proof}

Lemma~\ref{lem:local_stability} provides a step-wise stability bound. However, the Critical Optimization Radius (COR) from Section~\ref{sec:phenomenon_collapse} acts as a global threshold throughout the optimization process, remaining independent of step $t$. 
The COR is dictated by the most restrictive local stability bound, attained at what we identify as the bottleneck step $t^*$.

\begin{definition}[Critical Optimization Radius]
\label{def:cor}
We define the Critical Optimization Radius (COR) $\rho_{\text{critical}}$ as the minimum local stability bound over the optimization trajectory. Applying the curvature approximation $L_t \approx \lambda_{\max}(\mathbf{H}_t)$ from Assumption~\ref{assum:local_curvature}, the global COR is formulated as:
\begin{equation}
    \label{eq:cor_global_def}
    \rho_{\text{critical}} \triangleq \min_{t} \frac{\|\nabla \mathcal{L}_t\|}{\lambda_{\max}(\mathbf{H}_t)}.
\end{equation}

We define the \textit{bottleneck step} $t^*$ as the iteration achieving this minimum, where the local curvature restricts the perturbation radius the most:
\begin{equation}
    \label{eq:bottleneck_def}
    t^* \triangleq \underset{t}{\operatorname{argmin}} \frac{\|\nabla \mathcal{L}_t\|}{\lambda_{\max}(\mathbf{H}_t)}.
\end{equation}
\end{definition}

Equivalently, evaluating at the bottleneck step $t^*$ yields $\rho_{\text{critical}} = {\|\nabla \mathcal{L}_{t^*}\|}/{\lambda_{\max}(\mathbf{H}_{t^*})}$. 


\subsection{Bridging COR and Generalization via GSNR}
\label{sec:cor_gsnr_bridge}

While COR is derived from geometric constraints, we establish a formal connection between this stability threshold and the Gradient Signal-to-Noise Ratio (GSNR). Leveraging the Hessian-Covariance decomposition in \eqref{eq:hessian_identity}, we demonstrate that SAM's geometric stability couples directly with the statistical quality of the gradient.

\begin{theorem}[COR–GSNR Stability Decomposition]
\label{them:gsnr_bound}
The Critical Optimization Radius $\rho_{\text{critical}}$ at the bottleneck step $t^*$ is analytically factorized into a tripartite decomposition of local geometric flatness, model misspecification, and the GSNR.
Introducing the trace-based stable rank of the Hessian, $\kappa_s(\mathbf{H}_t) \triangleq \operatorname{Tr}(\mathbf{H}_t) / \lambda_{\max}(\mathbf{H}_t)$, we formulate the COR as:
\begin{equation}
    \label{eq:refined_bound}
    \rho_{\text{critical}} = \underbrace{\frac{\kappa_s(\mathbf{H}_{t^*})}{\sqrt{\operatorname{Tr}(\mathbf{H}_{t^*})}}}_{\text{Geometric Term}} \cdot \underbrace{\sqrt{1 + \frac{\operatorname{Tr}(\mathbf{\Xi}_{t^*})}{\operatorname{Tr}(\mathbf{H}_{t^*})}}}_{\text{Misspecification Term}} \cdot \underbrace{\sqrt{\frac{\text{GSNR}_{t^*}}{1 + \text{GSNR}_{t^*}}}}_{\text{Statistical Term}}
\end{equation}
where $\mathbf{\Xi}_{t^*}$ represents the residual matrix capturing model misspecification.
\end{theorem}

\begin{proof}
Taking the trace of the Hessian decomposition yields:
\begin{equation}
    \operatorname{Tr}(\mathbf{H}_{t^*}) \stackrel{\eqref{eq:hessian_identity}}{=} \operatorname{Tr}(\operatorname{Cov}[\mathbf{g}_{t^*}]) + \|\nabla \mathcal{L}_{t^*}\|^2 - \operatorname{Tr}(\mathbf{\Xi}_{t^*}).
\end{equation}
Substituting $\|\nabla \mathcal{L}_{t^*}\|^2 = \operatorname{Tr}(\operatorname{Cov}[\mathbf{g}_{t^*}]) \cdot \text{GSNR}_{t^*}$ based on the definition of GSNR, we isolate the covariance trace:
\begin{equation}
    \operatorname{Tr}(\operatorname{Cov}[\mathbf{g}_{t^*}]) = \frac{\operatorname{Tr}(\mathbf{H}_{t^*}) + \operatorname{Tr}(\mathbf{\Xi}_{t^*})}{1 + \text{GSNR}_{t^*}}.
    \label{eq:extracted_trace}
\end{equation}
Applying the trace-based stable rank $\kappa_s(\mathbf{H}_t) \triangleq \frac{\operatorname{Tr}(\mathbf{H}_t)}{\lambda_{\max}(\mathbf{H}_t)}$ to the global COR definition, the structural bound is derived as:
\begin{equation}
\begin{aligned}
    \rho_{\text{critical}} 
    &\stackrel{\eqref{eq:cor_global_def}}{=} \frac{\|\nabla \mathcal{L}_{t^*}\|}{\lambda_{\max}(\mathbf{H}_{t^*})} \\
    &\stackrel{\eqref{eq:gsnr_definition}}{=} \frac{\sqrt{\operatorname{Tr}(\operatorname{Cov}[\mathbf{g}_{t^*}]) \cdot \text{GSNR}_{t^*}}}{\lambda_{\max}(\mathbf{H}_{t^*})} \\
    &\stackrel{\eqref{eq:extracted_trace}}{=} \frac{\kappa_s(\mathbf{H}_{t^*})}{\operatorname{Tr}(\mathbf{H}_{t^*})} \cdot \sqrt{ \frac{\operatorname{Tr}(\mathbf{H}_{t^*}) + \operatorname{Tr}(\mathbf{\Xi}_{t^*})}{1 + \text{GSNR}_{t^*}} \cdot \text{GSNR}_{t^*} } \\
    &= \frac{\kappa_s(\mathbf{H}_{t^*})}{\sqrt{\operatorname{Tr}(\mathbf{H}_{t^*})}} \cdot \sqrt{1 + \frac{\operatorname{Tr}(\mathbf{\Xi}_{t^*})}{\operatorname{Tr}(\mathbf{H}_{t^*})}} \cdot \sqrt{\frac{\text{GSNR}_{t^*}}{1 + \text{GSNR}_{t^*}}}.
\end{aligned}
\end{equation}

\end{proof}

\begin{remark}\normalfont
This tripartite interplay reveals the underlying mechanics of optimization collapse. Detailed derivations are provided in the supplementary material.
\begin{itemize}
    \item Geometric term: It indicates that a flatter landscape, characterized by a higher stable rank $\kappa_s$, inherently accommodates larger perturbations. While SAM explicitly minimizes local sharpness, its execution relies on a sufficiently large COR, a prerequisite violated during Optimization Collapse.
    \item Misspecification term: It captures the model-data discrepancy. As the model converges, the discrepancy $\operatorname{Tr}(\mathbf{\Xi}_{t^*})$ approaches $0$, causing this term to asymptotically approach $1$. It primarily reflects the degree of convergence, independent of specific forgery categories. It fails to explain the dataset-specific nature of the collapse (e.g., its presence in NT and absence in FS), and
    therefore cannot be the cause of optimization collapse.
    \item Statistical term: It emerges as the primary driver of the collapse. Since $\sqrt{\mathrm{GSNR}_{t^*}/(1+\mathrm{GSNR}_{t^*})}$ is strictly monotonically increasing in $\mathrm{GSNR}_{t^*}$, any degradation in GSNR directly contracts the stability boundary. In the collapse regime under vanishing radii, where $\mathrm{GSNR}_{t^*} \ll 1$, the COR reduces to $\rho_{\mathrm{critical}} \propto \sqrt{\mathrm{GSNR}_{t^*}}$. This identifies degraded GSNR as the root cause of Optimization Collapse.
\end{itemize}

\end{remark}

In summary, when the model fails to extract a high-fidelity gradient signal, geometric stability cannot be sustained. Optimization Collapse is therefore a direct symptom of gradient signal degradation.

\subsection{Empirical Verification of Theorem~\ref{them:gsnr_bound}}
\label{sec:theoretical_verification}

To empirically validate Theorem~\ref{them:gsnr_bound}, we visualize the trajectories of GSNR across different forgery datasets and varying layers of the CLIP-ViT-L/14 backbone.
As illustrated in Fig.~\ref{fig:gsnr_analysis}, the GSNR evolution exhibits three distinct phases: 
(i) a Pre-optimization Phase, denoting the initial exploration for optimizable trajectories with a consistently low GSNR; 
(ii) an Optimization Phase, signifying effective optimization marked by a sharp increase in GSNR; 
and (iii) a Convergence Phase, exhibiting a gradual decay in GSNR as the model stabilizes.
Since optimization collapse represents a failure to transition into the second phase, the COR in \eqref{eq:refined_bound} is dictated by the GSNR at the bottleneck step $t^*$ located in the first phase.

We identify a consistent rank correlation between the $\text{GSNR}_{t^*}$ and the COR values shown in Fig.~\ref{fig:collapse_analysis}.
Specifically, in Fig.~\ref{fig:gsnr_analysis}(a), the $\text{GSNR}_{t^*}$ magnitudes follow the order $\text{DF} > \text{FS} > \text{F2F} \approx \text{NT}$, which closely mirrors the COR values observed in the 24-th layer in Fig.~\ref{fig:collapse_analysis}(b);
in Fig.~\ref{fig:gsnr_analysis}(b), we observe a decrease in $\text{GSNR}_{t^*}$ as the layer depth increases on the NT dataset, which corroborates the layer-wise decay trend of COR illustrated in Fig.~\ref{fig:collapse_analysis}(b).
Furthermore, we observe that a lower $\text{GSNR}_{t^*}$, which indicates poorer gradient quality, prolongs the search for a valid descent direction in the first phase.
This results in a delayed transition to the optimization phase, and the duration of this delay serves as an easily observable surrogate for the $\text{GSNR}_{t^*}$ ranking.
Collectively, these empirical evidences substantiate that 
$\text{GSNR}_{t^*}$ is the primary determinant of COR: the monotonic co-decay of GSNR and COR across both datasets and layer depths confirms the structural decomposition in Theorem~\ref{them:gsnr_bound}.

\begin{figure}[t]
    \centering
    \includegraphics[width=0.48\textwidth]{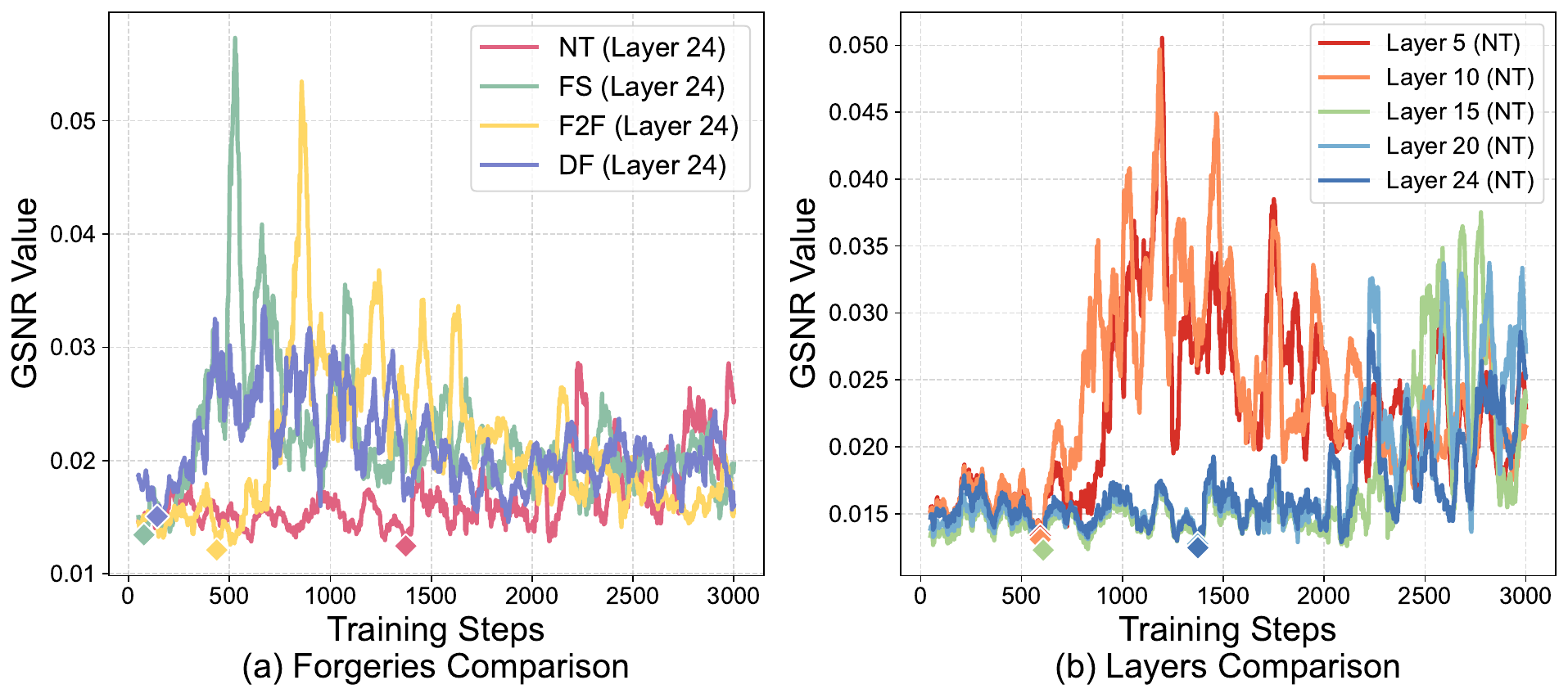}
    \caption{GSNR trajectories during optimization. 
    (a) GSNR values across four datasets at the 24-th layer. 
    (b) Layer-wise GSNR values on NeuralTextures. 
    Diamond markers denote the bottleneck step $t^*$.}
    \label{fig:gsnr_analysis}
\end{figure}

\subsection{Theoretical Insights}
\label{sec:theoretical_insights}
Theorem~\ref{them:gsnr_bound} unveils the intrinsic mechanism of Optimization Collapse: a vanishing COR (e.g., $\rho_{\text{critical}} \approx 0.001$ on NT) implies a diminished GSNR where the directional gradient signal is drowned out by high gradient variance, leading to poor generalization performance.
Therefore, rather than naively reducing the perturbation radius $\rho$, the fundamental solution requires elevating the fidelity of the gradient signal.

This theoretical insight also offers a guideline for our methodology design: by maximizing the GSNR, we lift the COR out of the collapse regime ($\rho_{\text{critical}} < 0.05$), thereby mitigating Optimization Collapse.
In practice, this equates to enhancing the strength and suppressing the variance of gradient signals.

\section{Methodology}
\label{sec:method}
Guided by the theoretical analysis in Section~\ref{sec:theoretical_insights}, we propose the \textbf{Co}ntrastive \textbf{R}egional \textbf{I}njection \textbf{T}ransformer (CoRIT) to mitigate optimization collapse. As shown in Fig.~\ref{fig:pipline}, CoRIT consists of three training-free strategies that aim to maximize the COR. We begin by deriving a computationally efficient Contrastive Gradient Proxy based on feature discrepancy. This proxy is used to construct a Region Refinement Mask that reduces gradient variance. The refined mask is then coupled with a Regional Signal Injection mechanism to counteract the layer-wise attenuation of non-semantic forensic signals. Finally, Hierarchical Representation Integration fuses intermediate signals with deep context to yield generalizable multi-scale representations.

\subsection{Contrastive Gradient Proxy}\label{subsec:gsnr_theory}

Maximizing the COR necessitates elevating the GSNR of the model parameters $\mathbf{w}$.
However, directly maximizing $\text{GSNR}(\mathbf{w})$ as an auxiliary optimization objective is computationally impractical; it requires per-sample backpropagation, which imposes an untenable overhead for large-scale architectures like CLIP-ViT-Large.
To circumvent this computational bottleneck, we shift our analytical focus from the parameter space to the feature space.
By the chain rule, the parameter gradients depend on the visual tokens $\mathbf{V}^{(l)}$ at layer $l$ via $\nabla_{\mathbf{w}} \mathcal{L} = (\partial \mathbf{V}^{(l)} / \partial \mathbf{w})^\top \nabla_{\mathbf{V}^{(l)}} \mathcal{L}$. This relationship implies that the statistical stability of the feature gradient $\nabla_{\mathbf{V}^{(l)}} \mathcal{L}$ serves as a necessary precondition for the stability of parameter gradients.
To enable efficient implementation, inspired by the Self-Blended Image (SBI) paradigm~\cite{DBLP:conf/cvpr/ShioharaY22}, we propose a gradient proxy to approximate gradient signals using feature-space contrastive discrepancies under a frozen backbone setting.

Let $\mathbf{V}^{(l)} \in \mathbb{R}^{N \times D}$ denote the visual tokens and $h(\cdot)$ the classifier head mapping features to the predictive logit $y$. 
The SBI~\cite{DBLP:conf/cvpr/ShioharaY22} training objective maximizes the logit discrepancy $\Delta y$ between the original visual tokens $\mathbf{V}^{(l)}_{\text{orig}}$ and their Self-Blended counterpart $\mathbf{V}^{(l)}_{\text{SBI}}$. 
A first-order Taylor expansion of $h$ gives:
\begin{equation}
    \Delta y = h(\mathbf{V}^{(l)}_{\text{SBI}}) - h(\mathbf{V}^{(l)}_{\text{orig}}) \approx \langle \nabla_{\mathbf{V}^{(l)}} h, \Delta \mathbf{V}^{(l)} \rangle,
\end{equation}
where $\Delta \mathbf{V}^{(l)} = \mathbf{V}^{(l)}_{\text{SBI}} - \mathbf{V}^{(l)}_{\text{orig}}$ denotes the feature discrepancy. 
According to the Cauchy-Schwarz inequality, $\Delta y$ is maximized when the gradient direction $\nabla_{\mathbf{V}^{(l)}} h$ aligns with the feature discrepancy $\Delta \mathbf{V}^{(l)}$.
We define the Contrastive Gradient Proxy (CGP) $\mathbf{g}_{\text{CGP}}^{(l)}$ as:
\begin{equation}
\mathbf{g}_{\text{CGP}}^{(l)} \triangleq \Delta \mathbf{V}^{(l)}.
\end{equation}

This proxy enables efficient GSNR estimation via two forward passes, eliminating backpropagation and the memory overhead of gradient storage.

\begin{figure*}[t] 
    \centering
    \includegraphics[width=1\textwidth]{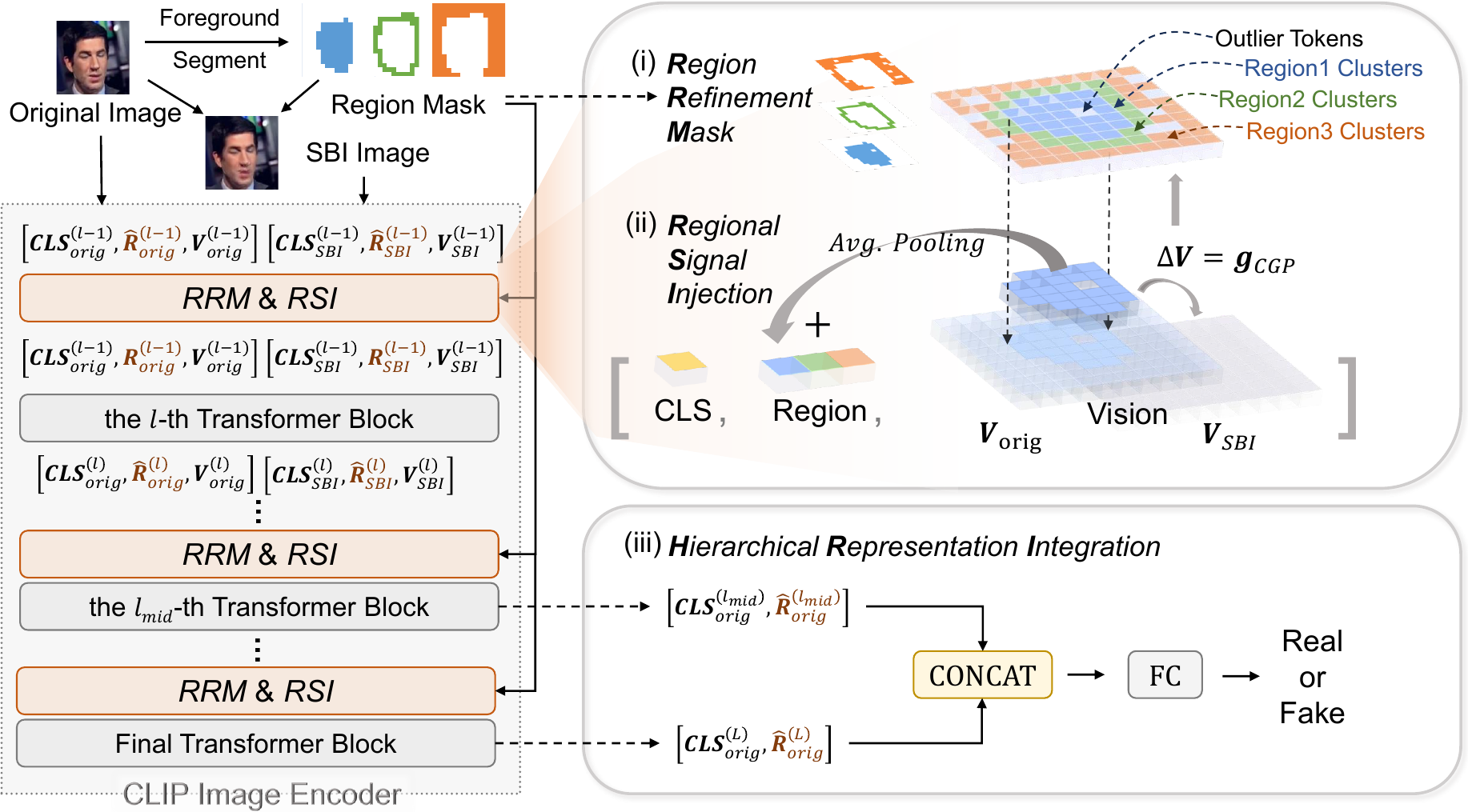}
    \caption{Overview of the proposed \textbf{Co}ntrastive \textbf{R}egional \textbf{I}njection \textbf{T}ransformer (CoRIT) pipeline. Given an original image and its Self-Blended Image (SBI) counterpart, both are fed through the frozen CLIP image encoder in parallel. The discrepancy of visual tokens serves as the Contrastive Gradient Proxy (CGP). At each transformer layer, CoRIT applies three training-free components: \textbf{(i)} The \textit{Region Refinement Mask} (RRM) clusters visual tokens around region anchors derived from the CGP, filtering out outlier tokens. \textbf{(ii)} The \textit{Regional Signal Injection} (RSI) aggregates the refined tokens via average pooling and injects them into additional Region Tokens through intra-layer residual connections. \textbf{(iii)} The \textit{Hierarchical Representation Integration} (HRI) concatenates the class token and region tokens from an intermediate layer $l_{\text{mid}}$ and the final layer $L$ for binary classification.}
    \label{fig:pipline} 
    \vspace{-10pt}
\end{figure*}

\subsection{Region Refinement Mask}
\label{subsec:rrm}

ViTs are prone to outlier tokens~\cite{DBLP:conf/iclr/DarcetOMB24} and domain-specific artifacts~\cite{DBLP:conf/icml/0002WJZLCYDW025}, which amplify gradient variance, leading to a degraded GSNR.
We introduce a Region Refinement Mask to calibrate the feature maps by focusing on spatial regions with high CGP coherence. 
A subsequent region-wise average pooling within these designated areas suppresses CGP variance without compromising representational diversity, securing a higher regional GSNR.

\subsubsection{Region Anchor Initialization}
Given a feature map consisting of $N$ visual tokens $\mathbf{V}^{(l)} = \{\mathbf{v}^{(l)}_i\}_{i=1}^N$, we employ $K$ pre-defined regions, where $\mathcal{N}_k$ denotes the set of token indices in the $k$-th region. 
For the $k$-th region at layer $l$, we compute a Region Anchor $\mathbf{c}^{(l)}_k$ to capture the dominant gradient direction. 
Using the CGP of visual tokens $\Delta \mathbf{v}^{(l)}_i = \mathbf{v}^{(l)}_{i, \text{SBI}} - \mathbf{v}^{(l)}_{i, \text{orig}}$, we formulate the anchor as the centroid of the gradients within the region:
\begin{equation}
    \mathbf{c}^{(l)}_k = \frac{1}{|\mathcal{N}_k|} \sum_{i \in \mathcal{N}_k} \Delta \mathbf{v}^{(l)}_i, \quad \mathbf{d}^{(l)}_k = \frac{\mathbf{c}^{(l)}_k}{\|\mathbf{c}^{(l)}_k\|_2},
\end{equation}
where $|\mathcal{N}_k|$ represents the number of tokens in region $k$, and $\mathbf{d}^{(l)}_k$ is the normalized directional vector of the region anchor.

\subsubsection{Region Refinement and Masking}
To filter outlier tokens, we refine the region $k$ by clustering tokens around the region anchor $\mathbf{c}^{(l)}_k$. We compute the projection $p^{(l)}_{k,i}$ of each token's proxy gradient onto this anchor direction: 
$p^{(l)}_{k,i} = \langle \Delta \mathbf{v}^{(l)}_i, \mathbf{d}^{(l)}_k \rangle$.
However, relying solely on this directional alignment is often insufficient, as we observe that the inherent noise in shallow layers induces cluster dispersion.
Inspired by visual contrastive decoding~\cite{DBLP:conf/cvpr/LengZCLLMB24}, we exclude tokens residing outside the pre-defined regions, thereby enhancing clustering stability.
We construct a binary region refinement mask $\hat{M}$ by thresholding these projections under this spatial constraint:
\begin{equation}
    \hat{M}^{(l)}_{k,i} = \mathbb{I}(p^{(l)}_{k,i} > \alpha \cdot \|\mathbf{c}^{(l)}_k\|_2) \cdot \mathbb{I}(i \in \mathcal{N}_k),
\end{equation}
where $\mathbb{I}(\cdot)$ is the indicator function, and $\alpha$ serves as a hyperparameter controlling the filtering threshold.

\subsubsection{Region Average Pooling}
To exploit average pooling as a noise reduction mechanism, we aggregate features within the refined regions to elevate the GSNR.
The region token $\mathbf{r}^{(l)}_k$ is computed by averaging the visual tokens $\mathbf{V}^{(l)}$ masked by $\hat{M}^{(l)}$:
\begin{equation}
    \label{eq:region_aggregation}
    \mathbf{r}^{(l)}_k = \frac{\sum_{i} \hat{M}^{(l)}_{k,i} \cdot \mathbf{v}^{(l)}_i}{\sum_{i} \hat{M}^{(l)}_{k,i} + \epsilon},
\end{equation}
where $\mathbf{v}^{(l)}_i$ denotes the $i$-th visual token in layer $l$, and $\epsilon=10^{-6}$ ensures numerical stability.

\subsection{Regional Signal Injection}
\label{subsec:rsi}
The layer-wise COR attenuation observed on NT and F2F in Fig.~\ref{fig:collapse_analysis}(b) reveals that the GSNR of non-semantic signals progressively decays.
We propose the Regional Signal Injection mechanism to counteract this degradation. This mechanism iteratively injects the high-GSNR regional features identified by RRM into a set of additional Region Tokens $\mathbf{R} \in \mathbb{R}^{K \times D}$, preserving generalizable non-semantic forgery representations within the final layer tokens.

\subsubsection{Iterative Joint Modeling} 
We initialize the region tokens as $\mathbf{R}^{(0)} = \mathbf{0}$. At layer $l$, the preceding region tokens $\mathbf{R}^{(l-1)}$ are concatenated with the class token $\mathbf{CLS}^{(l-1)}$ and visual tokens $\mathbf{V}^{(l-1)}$. The Transformer block $\mathcal{T}_l$ processes this joint sequence:
\begin{equation}
    [\mathbf{CLS}^{(l)}, \hat{\mathbf{R}}^{(l)}, \mathbf{V}^{(l)}] = \mathcal{T}_l \left( [\mathbf{CLS}^{(l-1)}, \mathbf{R}^{(l-1)}, \mathbf{V}^{(l-1)}] \right).
    \label{eq:joint_model}
\end{equation}

\subsubsection{Intra-Layer Residual Injection} 
To preserve generalizable signal strength across layers, we update the intermediate region tokens $\hat{\mathbf{R}}^{(l)}_k$ by injecting the high-GSNR regional features $\mathbf{r}^{(l)}_k$ defined in Eq.~\eqref{eq:region_aggregation} via an intra-layer residual connection:
\begin{equation}
    \mathbf{R}^{(l)}_k = \hat{\mathbf{R}}^{(l)}_k + \mathbf{r}^{(l)}_k.
    \label{eq:recurrent_update}
\end{equation}

The updated $\mathbf{R}^{(l)}_k$ serves as the input for the subsequent layer $l+1$ in Eq.~\eqref{eq:joint_model}.

\subsection{Hierarchical Representation Integration}
\label{subsec:hri}
Empirical observations in Fig.~\ref{fig:collapse_analysis}(b) suggest that middle layers capture more generalizable non-semantic forensic artifacts. We introduce a simple yet effective strategy, Hierarchical Representation Integration, to explicitly fuse middle-layer signals with deep-layer context.

We define the forensic representation as the concatenation of the class token and the region tokens: $\mathbf{F}^{(l)} = [\mathbf{CLS}^{(l)}, \mathbf{R}^{(l)}]$. 
The final output integrates representations from a predefined intermediate layer $l_{\text{mid}}$ and the final layer $L$:
\begin{equation}
    \mathbf{F}_{\text{out}} = \text{Concat}(\mathbf{F}^{(l_{\text{mid}})}, \mathbf{F}^{(L)}).
\end{equation}

By preserving non-semantic forgery representations with high GSNR, $\mathbf{F}_{\text{out}}$ provides a robust input for final classification and downstream tasks.

\begin{table}[t]
    \centering
    \caption{Cross-domain performance comparison on classical  forensic benchmarks (Frame-Level AUC (\%)). $\dagger$ denotes methods using a frozen backbone with a lightweight learnable classifier. Methods below the midrule share the same CLIP-ViT-L/14 backbone and SAM optimizer. Best results are in \textbf{bold}.}
    \label{tab:ff_cross}
    
    \resizebox{0.48\textwidth}{!}{

    \begin{tabular}{l|c|cccccc}
        \toprule
        Method & Venues & CDFv1 & CDFv2 & DFDCP & DFDC & DFD & Avg. \\
        \midrule
        Xception~\cite{DBLP:conf/iccv/RosslerCVRTN19} & ICCV'19 & 77.9 & 73.7 & 73.7 & 70.8 & 81.6 & 75.5 \\
        EfficientB4~\cite{DBLP:conf/icml/TanL19} & ICML'19 & 79.1 & 74.9 & 72.8 & 69.6 & 81.5 & 75.6 \\
        FWA~\cite{DBLP:conf/cvpr/LiL19c} & CVPRW'19 & 79.0 & 66.8 & 63.8 & - & 74.0 & - \\
        F3Net~\cite{DBLP:conf/eccv/QianYSCS20} & AAAI'20 & 77.7 & 73.5 & 73.5 & 70.2 & 79.8 & 74.9 \\
        Face X-ray~\cite{DBLP:conf/cvpr/LiBZYCWG20} & CVPR'20 & 70.9 & 67.9 & 69.4 & 63.3 & 76.6 & 69.6 \\
        FFD~\cite{DBLP:conf/cvpr/DangLS0020} & CVPR'20 & 78.4 & 74.4 & 74.3 & 70.3 & 80.2 & 75.5 \\
        SPSL~\cite{DBLP:conf/cvpr/LiuLZCH0ZY21} & CVPR'21 & 81.5 & 76.5 & 74.1 & 70.4 & 81.2 & 76.7 \\
        SRM~\cite{DBLP:conf/cvpr/LuoZY021} & CVPR'21 & 79.3 & 75.5 & 74.1 & 70.0 & 81.2 & 76.0 \\
        Recce~\cite{DBLP:conf/cvpr/Cao0YCDY22} & CVPR'22 & 76.8 & 73.2 & 73.4 & 71.3 & 81.2 & 75.2 \\
        SBI~\cite{DBLP:conf/cvpr/ShioharaY22} & CVPR'22 & - & 81.3 & 79.9 & - & 77.4 & - \\
        UCF~\cite{DBLP:conf/iccv/0002ZFW23} & ICCV'23 & 77.9 & 75.3 & 75.9 & 71.9 & 80.7 & 76.3 \\
        ED~\cite{DBLP:conf/aaai/BaLLWLL024} & AAAI'24 & 81.8 & 86.4 & 85.1 & 72.1 & - & - \\
        FoCus~\cite{DBLP:journals/tifs/TianCYFWDH24} & TIFS'24 & - & 72.0 & 77.8 & 66.9 & - & - \\
        LSDA~\cite{DBLP:conf/cvpr/0002LLLW24} & CVPR'24 & 86.7 & 83.0 & 81.5 & 73.6 & 88.0 & 82.6 \\
        DiffFake~\cite{DBLP:conf/nips/0016CY0SDJ24} & NIPS'24 & - & 80.5 & 81.0 & - & 90.4 & - \\
        UDD~\cite{DBLP:conf/aaai/FuYYCL25} & AAAI'25 & - & 86.9 & 85.6 & 75.8 & 91.0 & - \\
        KND~\cite{liu2025knowledge} & MM'25 & \textbf{91.6} &   81.5 &  80.3 &  78.2 &  87.2 &  83.8 \\
        FreqDebias~\cite{DBLP:conf/cvpr/KashianiTA25} & CVPR'25 & 87.5 & 83.6 & 86.8 & 82.4 & 74.1 & 82.9 \\
        ED$^4$~\cite{DBLP:journals/tip/ChengZZYLWL25} & TIP'25 & 88.7 & 83.9 & - & 74.3 & 87.7 & - \\
        FIA-USA~\cite{DBLP:journals/corr/abs-2504-04827} & NIPS'25 & 90.1 & 86.7 & 81.8 & - & 82.1 & - \\
        \midrule 
        UniFD$^{\dagger}$~\cite{DBLP:conf/cvpr/OjhaLL23} & CVPR'23 & 74.3 & 74.9 & 75.5 & 73.6 & 86.0 & 76.9 \\
        ForAda~\cite{DBLP:conf/cvpr/CuiLLZD25} & CVPR'25 & 83.7 & 81.0 & 81.6 & 81.6 & 90.9 & 83.8 \\
        \textbf{Ours}$^{\dagger}$ & - & 90.9 & \textbf{89.1} & \textbf{88.8} & \textbf{84.5} & \textbf{91.3} & \textbf{88.9} \\
        \bottomrule
    \end{tabular}
    }
\end{table}

\begin{figure}[t]
    \centering
    \includegraphics[width=\linewidth]{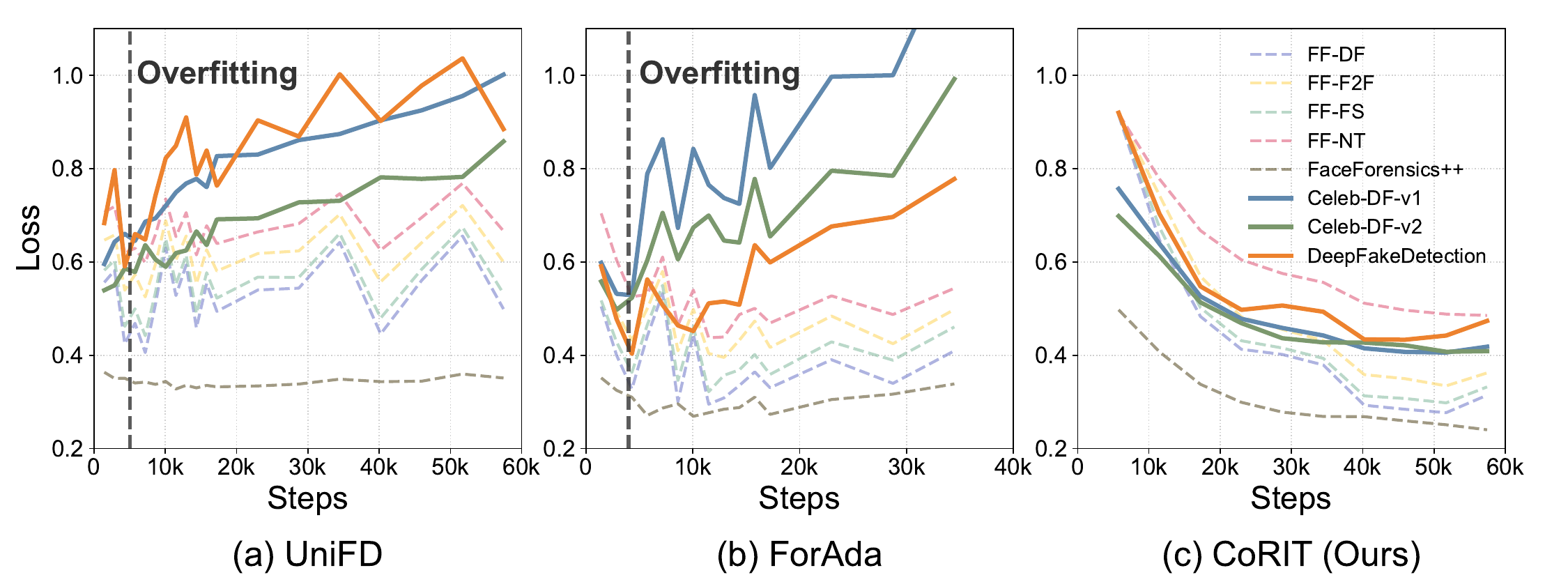}
    \vspace{-10pt}
    \caption{Training loss curves of compared methods illustrating generalization stability. Dashed lines indicate intra-dataset evaluation; solid lines indicate cross-dataset evaluation.}
    \label{fig:loss_compare}
\end{figure}

\begin{table}[t]
    \centering
    \caption{Cross-domain performance on classical  forensic benchmarks (Video-Level AUC (\%)). Best results are in \textbf{bold}.}
    \label{tab:ff_cross_video}
    
    \resizebox{0.42\textwidth}{!}{
        \begin{tabular}{l|c|cccc}
            \toprule
            Method & Venues & CDFv2 & DFDC & DFDCP & Avg. \\
            \midrule
            LipForensics~\cite{DBLP:conf/cvpr/HaliassosVPP21} & CVPR'21 & 82.4 & 73.5 & - & - \\
            FTCN~\cite{DBLP:conf/iccv/ZhengB0ZW21} & ICCV'21 & 86.9 & - & 74.0 & - \\
            PCL + I2G~\cite{DBLP:conf/iccv/ZhaoXXDXX21} & ICCV'21 & 90.0 & 67.5 & 74.3 & 77.3 \\
            HCIL~\cite{DBLP:conf/eccv/GuYCDM22} & ECCV'22 & 79.0 & 69.2 & - & - \\
            SBI~\cite{DBLP:conf/cvpr/ShioharaY22} & CVPR'22 & 92.8 & 71.9 & 85.5 & 83.4 \\
            3DDet.~\cite{DBLP:journals/pami/ZhuFZZZLL23} & TPAMI'23 & 72.0 & 72.4 & - & - \\
            LSDA~\cite{DBLP:conf/cvpr/0002LLLW24} & CVPR'24 & 91.1 & 77.0 & - & - \\
            CFM~\cite{DBLP:journals/tifs/LuoKHHKK24} & TIFS'24 & 89.7 & - & 80.2 & - \\
            LAA-Net~\cite{DBLP:conf/cvpr/NguyenMSKA0GA24} & CVPR'24 & 84.0 & - & 74.1 & - \\
            FIA-USA~\cite{DBLP:journals/corr/abs-2504-04827} & NIPS'25 & 83.9 & 74.3 & - & - \\
            Effort~\cite{DBLP:conf/icml/0002WJZLCYDW025} & ICML'25 & \textbf{95.6} & 84.3 & 90.9 & 90.3 \\
            \midrule
            \textbf{Ours} & - & 94.1 & \textbf{87.6} & \textbf{91.3} & \textbf{91.0} \\
            \bottomrule
        \end{tabular}
    }
\end{table}

\begin{table}[t]
    \centering
    \caption{Cross-domain performance on five representative face forgery types in DF40 (Frame-Level AUC (\%)). Best results are in \textbf{bold}.}
    \label{tab:df40_cross}
    
    \resizebox{0.48\textwidth}{!}{
        \begin{tabular}{l|c|cccccc}
            \toprule
            Method & Venues & uniface & e4s & facedancer & fsgan & inswap & Avg. \\
            \midrule
            Recce~\cite{DBLP:conf/cvpr/Cao0YCDY22} & CVPR'22 & 84.2 & 65.2 & 78.3 & 88.4 & 79.5 & 79.1 \\
            SBI~\cite{DBLP:conf/cvpr/ShioharaY22} & CVPR'22 & 64.4 & 69.0 & 44.7 & 87.9 & 63.3 & 65.9 \\
            CORE~\cite{DBLP:conf/cvpr/NiMYQRZ22} & CVPRW'22 & 81.7 & 63.4 & 71.7 & 91.1 & 79.4 & 77.5 \\
            IID~\cite{DBLP:conf/cvpr/HuangWYA00Y23} & CVPR'23 & 79.5 & 71.0 & 79.0 & 86.4 & 74.4 & 78.1 \\
            UCF~\cite{DBLP:conf/iccv/0002ZFW23} & ICCV'23 & 78.7 & 69.2 & 80.0 & 88.1 & 76.8 & 78.6 \\
            LSDA~\cite{DBLP:conf/cvpr/0002LLLW24} & CVPR'24 & 85.4 & 68.4 & 75.9 & 83.2 & 81.0 & 78.8 \\
            CDFA~\cite{DBLP:conf/eccv/LinSLLNCL24} & ECCV'24 & 76.5 & 67.4 & 75.4 & 84.8 & 72.0 & 75.2 \\
            ProgressiveDet~\cite{DBLP:conf/nips/Cheng0ZL0L24} & NIPS'24 & 84.5 & 71.0 & 73.6 & 86.5 & 78.8 & 78.9 \\
            FIA-USA~\cite{DBLP:journals/corr/abs-2504-04827} & NIPS'25 & \textbf{91.8} & 87.5 & 83.0 & 86.3 & 87.4 & 87.2 \\
            \midrule 
            UniFD~\cite{DBLP:conf/cvpr/OjhaLL23} & CVPR'23 & 82.7 & 91.4 & 77.9 & 85.6 & 76.5 & 80.0 \\
            \textbf{Ours} & - & \textbf{91.8} & \textbf{94.4} & \textbf{89.2} & \textbf{94.2} & \textbf{88.4} & \textbf{91.6} \\
            \bottomrule
        \end{tabular}
    }
\end{table}

\begin{table*}[t]
\centering
\caption{Generalization performance (mAP (\%)) on next-generation generative benchmarks. Best and second-best results are in \textbf{bold} and \underline{underline}.}
\label{tab:unidf_ap}
\resizebox{\textwidth}{!}{%
\begin{tabular}{l|cccccc|c|cc|cc|c|ccc|ccc|c|c}
\toprule
\multirow{3}{*}{Method} & \multicolumn{6}{c|}{GAN} & \multirow{3}{*}{\shortstack{Deep\\fakes}} & \multicolumn{2}{c|}{Low level} & \multicolumn{2}{c|}{Perceptual loss} & \multirow{3}{*}{Guided} & \multicolumn{3}{c|}{LDM} & \multicolumn{3}{c|}{Glide} & \multirow{3}{*}{Dalle} & \multirow{3}{*}{mAP} \\
\cmidrule{2-7} \cmidrule{9-10} \cmidrule{11-12} \cmidrule{14-16} \cmidrule{17-19}
 & Pro- & Cycle- & Big- & Style- & Gau- & Star- & & \multirow{2}{*}{SITD} & \multirow{2}{*}{SAN} & \multirow{2}{*}{CRN} & \multirow{2}{*}{IMLE} & & 200 & 200 & 100 & 100 & 50 & 100 & & \\
 & GAN & GAN & GAN & GAN & GAN & GAN & & & & & & & steps & w/cfg & steps & 27 & 27 & 10 & & \\
\midrule
\midrule
CNN-Spot~\cite{DBLP:conf/cvpr/WangW0OE20} & \textbf{100.0} & 93.47 & 84.50 & 99.54 & 89.49 & 98.15 & 89.02 & 73.75 & 59.47 & \underline{98.24} & \underline{98.40} & 73.72 & 70.62 & 71.00 & 70.54 & 80.65 & 84.91 & 82.07 & 70.59 & 83.58 \\
PatchForensics~\cite{DBLP:conf/eccv/ChaiBLI20} & 80.88 & 72.84 & 71.66 & 85.75 & 65.99 & 69.25 & 76.55 & 76.19 & \underline{76.34} & 74.52 & 68.52 & 75.03 & 87.10 & 86.72 & 86.40 & 85.37 & 83.73 & 78.38 & 75.67 & 77.73 \\
Co-occurrence~\cite{DBLP:conf/mediaforensics/NatarajMMCFBR19} & 99.74 & 80.95 & 50.61 & 98.63 & 53.11 & 67.99 & 59.14 & 68.98 & 60.42 & 73.06 & 87.21 & 70.20 & 91.21 & 89.02 & 92.39 & 89.32 & 88.35 & 82.79 & 80.96 & 78.11 \\
Freq-spec~\cite{DBLP:conf/wifs/0022KC19} & 55.39 & \textbf{100.0} & 75.08 & 55.11 & 66.08 & \textbf{100.0} & 45.18 & 47.46 & 57.12 & 53.61 & 50.98 & 57.72 & 77.72 & 77.25 & 76.47 & 68.58 & 64.58 & 61.92 & 67.77 & 66.21 \\
F3Net~\cite{DBLP:conf/eccv/QianYSCS20} & \underline{99.96} & 84.32 & 69.90 & 99.72 & 56.71 & \textbf{100.0} & 78.82 & 52.89 & 46.70 & 63.39 & 64.37 & 70.53 & 73.76 & 81.66 & 74.62 & 89.81 & 91.04 & 90.86 & 71.84 & 76.89 \\
LGrad~\cite{tan2023learning} & \textbf{100.0} & 93.98 & 90.69 & 99.86 & 79.36 & \underline{99.98} & 67.91 & 59.42 & 51.42 & 63.52 & 69.61 & 87.06 & 99.03 & 99.16 & 99.18 & 93.23 & 95.10 & 94.93 & 97.23 & 86.35 \\
FreqNet~\cite{DBLP:conf/aaai/Tan0WGLW24} & 99.92 & \underline{99.63} & \underline{96.05} & 99.89 & \underline{99.71} & 98.63 & \textbf{99.92} & 94.42 & 74.59 & 80.10 & 75.70 & 96.27 & 96.06 & \textbf{100.0} & 62.34 & \underline{99.80} & \underline{99.78} & 96.39 & 77.78 & 91.95 \\
NPR~\cite{DBLP:conf/cvpr/TanLZWGLW24} & \textbf{100.0} & 99.53 & 94.53 & \textbf{99.94} & 88.82 & \textbf{100.0} & 84.41 & \textbf{97.95} & \textbf{99.99} & 50.16 & 50.16 & \textbf{98.26} & \textbf{99.92} & \underline{99.91} & \textbf{99.92} & \textbf{99.87} & \textbf{99.89} & \textbf{99.92} & \underline{99.26} & \underline{92.76} \\

\midrule
UniFD~\cite{DBLP:conf/cvpr/OjhaLL23} & \textbf{100.0} & 98.13 & 94.46 & 86.66 & 99.25 & 99.53 & 91.67 & 78.54 & 67.54 & 83.12 & 91.06 & 79.24 & 95.81 & 79.77 & 95.93 & 93.93 & 95.12 & 94.59 & 88.45 & 90.14 \\
\textbf{Ours} & \textbf{100.0} & \textbf{100.0} & \textbf{99.96} & \underline{99.91} & \textbf{100.0} & \textbf{100.0} & \underline{92.24} & \underline{97.14} & 65.89 & \textbf{99.95} & \textbf{100.0} & \underline{96.31} & \underline{99.81} & 98.30 & \underline{99.81} & 99.08 & 99.19 & \underline{99.00} & \textbf{99.56} & \textbf{97.17} \\
\bottomrule
\end{tabular}%
}
\end{table*}

\begin{table*}[t]
\centering
\caption{Generalization performance (mAcc (\%)) on next-generation generative benchmarks. Best and second-best results are in \textbf{bold} and \underline{underline}.}
\label{tab:unidf_acc}
\resizebox{\textwidth}{!}{
\begin{tabular}{l|cccccc|c|cc|cc|c|ccc|ccc|c|c}
\toprule
\multirow{3}{*}{Method} & \multicolumn{6}{c|}{GAN} & \multirow{3}{*}{\shortstack{Deep\\fakes}} & \multicolumn{2}{c|}{Low level} & \multicolumn{2}{c|}{Perceptual loss} & \multirow{3}{*}{Guided} & \multicolumn{3}{c|}{LDM} & \multicolumn{3}{c|}{Glide} & \multirow{3}{*}{Dalle} & \multirow{3}{*}{mAcc} \\
\cmidrule{2-7} \cmidrule{9-10} \cmidrule{11-12} \cmidrule{14-16} \cmidrule{17-19}
 & Pro- & Cycle- & Big- & Style- & Gau- & Star- & & \multirow{2}{*}{SITD} & \multirow{2}{*}{SAN} & \multirow{2}{*}{CRN} & \multirow{2}{*}{IMLE} & & 200 & 200 & 100 & 100 & 50 & 100 & & \\
 & GAN & GAN & GAN & GAN & GAN & GAN & & & & & & & steps & w/cfg & steps & 27 & 27 & 10 & & \\
\midrule
\midrule
CNN-Spot~\cite{DBLP:conf/cvpr/WangW0OE20} & \underline{99.99} & 85.20 & 70.20 & 85.70 & 78.95 & 91.70 & 53.47 & 66.67 & 48.69 & \underline{86.31} & \underline{86.26} & 60.07 & 54.03 & 54.96 & 54.14 & 60.78 & 63.80 & 65.66 & 55.58 & 69.58 \\
PatchForensics~\cite{DBLP:conf/eccv/ChaiBLI20} & 75.03 & 68.97 & 68.47 & 79.16 & 64.23 & 63.94 & 75.54 & 75.14 & \underline{75.28} & 72.33 & 55.30 & 67.41 & 76.50 & 76.10 & 75.77 & 74.81 & 73.28 & 68.52 & 67.91 & 71.24 \\
Co-occurrence~\cite{DBLP:conf/mediaforensics/NatarajMMCFBR19} & 97.70 & 63.15 & 53.75 & 92.50 & 51.10 & 54.70 & 57.10 & 63.06 & 55.85 & 65.65 & 65.80 & 60.50 & 70.70 & 70.55 & 71.00 & 70.25 & 69.60 & 69.90 & 67.55 & 66.86 \\
Freq-spec~\cite{DBLP:conf/wifs/0022KC19} & 49.90 & \textbf{99.90} & 50.50 & 49.90 & 50.30 & 99.70 & 50.10 & 50.00 & 48.00 & 50.60 & 50.10 & 50.90 & 50.40 & 50.40 & 50.30 & 51.70 & 51.40 & 50.40 & 50.00 & 55.45 \\
F3Net~\cite{DBLP:conf/eccv/QianYSCS20} & 99.38 & 76.38 & 65.33 & 92.56 & 58.10 & \textbf{100.0} & 63.48 & 54.17 & 47.26 & 51.47 & 51.47 & 69.20 & 68.15 & 75.35 & 68.80 & 81.65 & 83.25 & 83.05 & 66.30 & 71.33 \\
LGrad~\cite{tan2023learning} & 99.84 & 85.39 & 82.88 & 94.83 & 72.45 & 99.62 & 58.00 & 62.50 & 50.00 & 50.74 & 50.78 & 77.50 & 94.20 & 95.85 & 94.80 & 87.40 & 90.70 & \underline{89.55} & \underline{88.35} & 80.28 \\
FreqNet~\cite{DBLP:conf/aaai/Tan0WGLW24} & 97.90 & 95.84 & 90.45 & \textbf{97.55} & 90.24 & 93.41 & \textbf{97.40} & \textbf{88.92} & 59.04 & 71.92 & 67.35 & \textbf{86.70} & 84.55 & \textbf{99.58} & 65.56 & 85.69 & \textbf{97.40} & 88.15 & 59.06 & 85.09 \\
NPR~\cite{DBLP:conf/cvpr/TanLZWGLW24} & 99.84 & 95.00 & 87.55 & \underline{96.23} & 86.57 & \underline{99.75} & \underline{76.89} & 66.94 & \textbf{98.63} & 50.00 & 50.00 & \underline{84.55} & \textbf{97.65} & \underline{98.00} & \textbf{98.20} & \textbf{96.25} & \underline{97.15} & \textbf{97.35} & 87.15 & \underline{87.56} \\

\midrule
UniFD~\cite{DBLP:conf/cvpr/OjhaLL23} & \textbf{100.0} & 98.50 & \underline{94.50} & 82.00 & \textbf{99.50} & 97.00 & 66.60 & 63.00 & 57.50 & 59.50 & 72.00 & 70.03 & 94.19 & 73.76 & 94.36 & 79.07 & 79.85 & 78.14 & 86.78 & 81.38 \\
\textbf{Ours} & \underline{99.99} & \underline{99.70} & \textbf{98.95} & 93.33 & \underline{99.20} & \textbf{100.0} & 51.90 & \underline{87.22} & 59.36 & \textbf{97.35} & \textbf{97.35} & 78.25 & \underline{97.00} & 86.60 & \underline{96.55} & \underline{88.90} & 91.25 & 88.65 & \textbf{95.35} & \textbf{89.84} \\
\bottomrule
\end{tabular}%
}
\end{table*}

\begin{table}[t]
    \centering
    \caption{Cross-manipulation performance in FF++ dataset (Frame-Level AUC (\%)). Best results are in \textbf{bold}.}
    \label{tab:cross_manipulation}
    
    \setlength{\tabcolsep}{4pt}
    
    \resizebox{0.65\columnwidth}{!}{
        \begin{tabular}{lccccc} 
            \toprule
            \multirow{2}{*}{Method} & \multicolumn{4}{c}{Train on remaining three} & \multirow{2}{*}{Avg} \\
            \cmidrule(lr){2-5} 
             & DF & F2F & FS & NT & \\
            \midrule
            UniFD~\cite{DBLP:conf/cvpr/OjhaLL23} & 95.8 & 75.0 & 90.1 & 59.4 & 80.1 \\
            ForAda~\cite{DBLP:conf/cvpr/CuiLLZD25} & 97.0 & 83.3 & 92.0 & 64.4 & 84.2 \\
            \textbf{Ours} & \textbf{98.3} & \textbf{87.5} & \textbf{93.6} & \textbf{67.6} & \textbf{86.8} \\
            \bottomrule
        \end{tabular}%
    }
\end{table}

\begin{table}[t]
    \centering
    \caption{Layer-wise COR comparison between UniFD and our CoRIT on the NT dataset. Best results are in \textbf{bold}.}
    \label{tab:layer_cor}
    
    \resizebox{0.7\columnwidth}{!}{%
        \begin{tabular}{l|ccccc}
            \toprule
            Method & L5 & L10 & L15 & L20 & L24 \\
            \midrule
            UniFD~\cite{DBLP:conf/cvpr/OjhaLL23} & 0.57 & 0.49  & 0.04  & 0.03 & 0.01 \\
            \textbf{Ours} & \textbf{0.60} & \textbf{2.67} & \textbf{2.60}  & \textbf{0.40} & \textbf{0.38} \\
            \bottomrule
        \end{tabular}%
    }
\end{table}

\begin{figure}[t]
    \centering
    \includegraphics[width=0.4\textwidth]{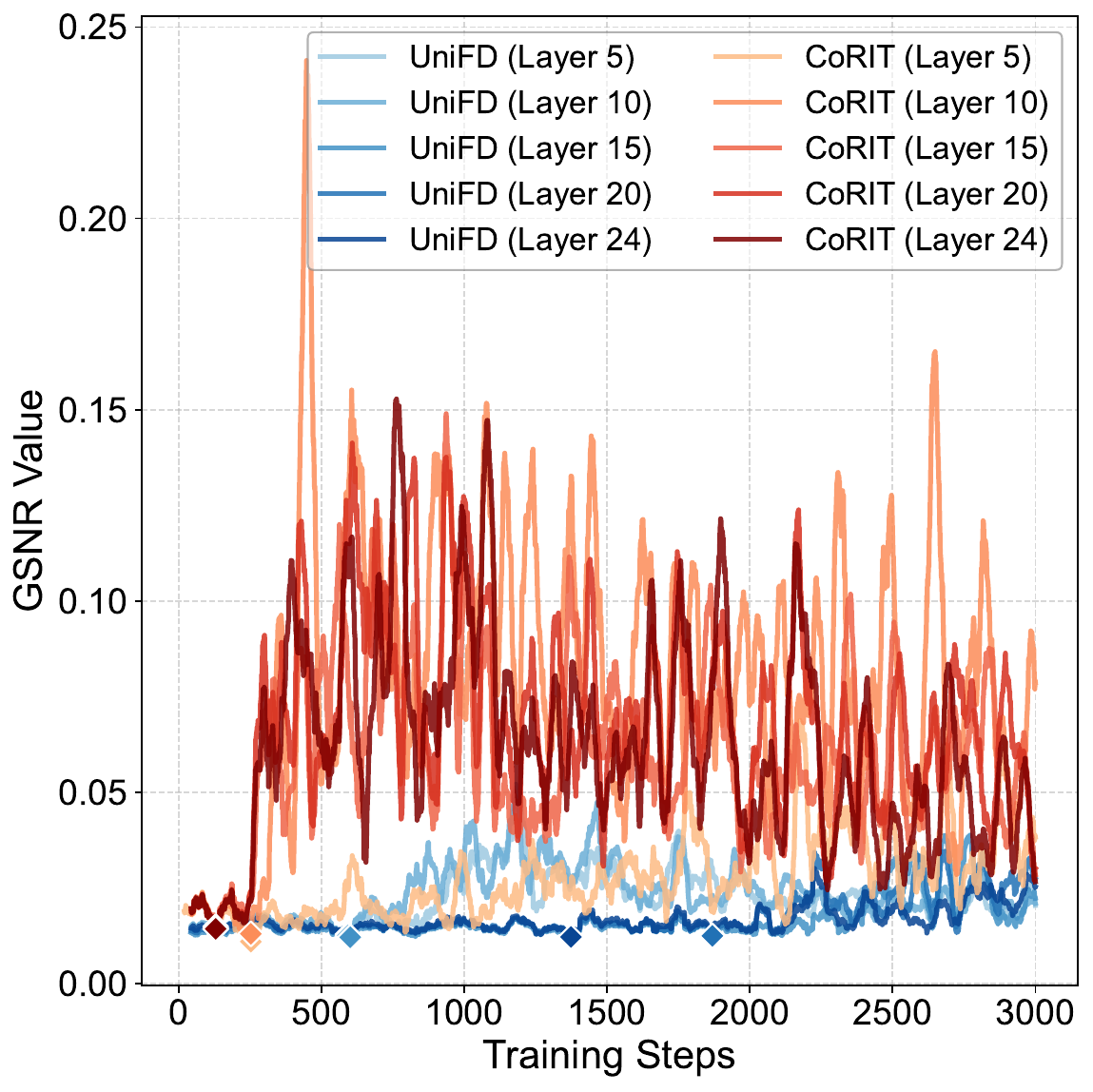}
    \caption{Comparison of layer-wise GSNR curves during training between UniFD and CoRIT on the NT dataset.}
    \label{fig:gsnr_comparison}
\end{figure}

\begin{figure}[t]
    \centering
    \includegraphics[width=0.45\textwidth]{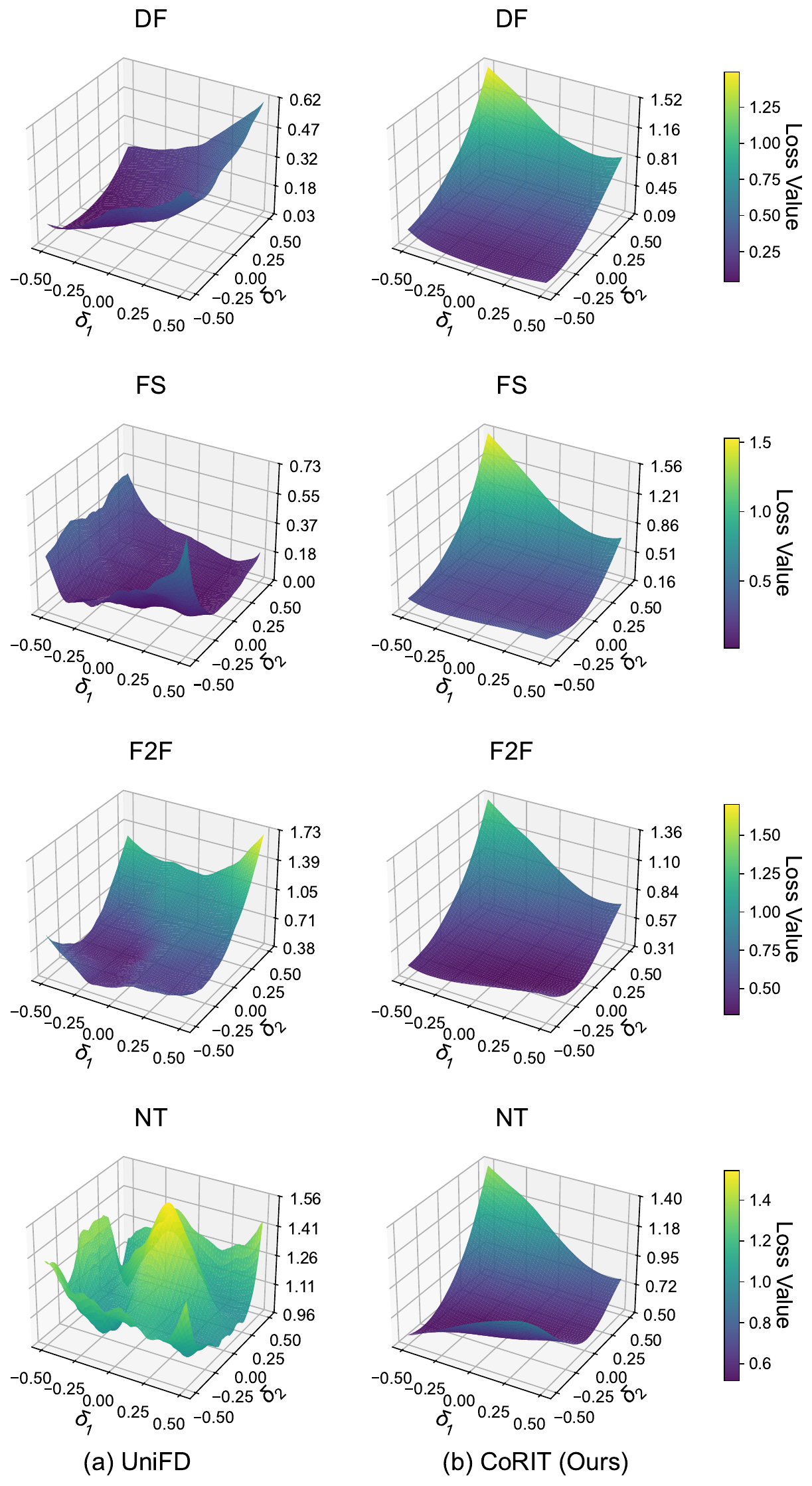}
    \caption{Comparison of loss landscapes between UniFD and the proposed CoRIT on the NT dataset.}
    \label{fig:lossland}
\end{figure}

\section{Experiments}
\subsection{Datasets and Evaluation Protocols}
\label{subsec:protocol}
To comprehensively assess the generalization of CoRIT across both classical facial forgery and emerging AI-generated content, we employ three diverse evaluation benchmarks.

\subsubsection{Classical Forensic Benchmarks.} We utilize the FaceForensics++ (FF++)~\cite{DBLP:conf/iccv/RosslerCVRTN19} dataset, which contains four manipulation types: Identity Replacement (DeepFakes, FaceSwap), Expression Reenactment (Face2Face), and Neural Rendering (NeuralTextures).
For cross-manipulation evaluation, the model is trained on three manipulation types and tested on the held-out type.
For cross-dataset evaluation, we train on FF++ and test on several widely used datasets: Deepfake Detection Challenge (DFDC)~\cite{deepfake-detection-challenge}, its preview version (DFDCP)~\cite{DBLP:journals/corr/abs-1910-08854}, DeepFake Detection (DFD)~\cite{DFD2019}, and both iterations of the Celeb-DF benchmark (CDF-v1 and CDF-v2)~\cite{DBLP:conf/cvpr/LiYSQL20}. 
All data preprocessing and evaluation protocols follow the standards of DeepfakeBench~\cite{DBLP:conf/nips/00020YLW23} to ensure reproducibility and fair comparison. We report Frame-level and Video-level AUC (Area Under the Curve) as the primary metrics.

\subsubsection{Next-Generation Generative Benchmarks.} To further test against rapidly evolving face forgery techniques, we incorporate the DF40 benchmark~\cite{DBLP:conf/nips/YanYCZFZLWDWY24}, a large-scale benchmark covering 40 distinct generation methods. Following prior work~\cite{DBLP:journals/corr/abs-2504-04827}, we select representative face forgery types and adopt Protocol-2 (Cross-Domain Evaluation) on the CDF domain. We report Frame-level AUC as the evaluation metric.

\subsubsection{UniversalFakeDetect Benchmarks.} Distinct from the face-centric evaluations above, we adopt the UniversalFakeDetect~\cite{DBLP:conf/cvpr/OjhaLL23} benchmark to evaluate on general natural scenes, which covers entire synthesized images paired with real images from LAION~\cite{DBLP:journals/corr/abs-2111-02114} and ImageNet~\cite{DBLP:journals/ijcv/RussakovskyDSKS15}.
Following standard protocols~\cite{DBLP:conf/cvpr/WangW0OE20,DBLP:conf/cvpr/OjhaLL23}, we train on the ForenSynths~\cite{DBLP:conf/cvpr/WangW0OE20} training set (ProGAN~\cite{DBLP:conf/iclr/KarrasALL18}). The test set comprises 13 generative models, spanning from GAN families (e.g., CycleGAN~\cite{DBLP:conf/iccv/ZhuPIE17}, StyleGAN~\cite{DBLP:journals/pami/KarrasLA21}, StarGAN~\cite{DBLP:conf/cvpr/ChoiCKH0C18}) to recent diffusion models (e.g., ADM~\cite{DBLP:conf/nips/DhariwalN21}, DALL-E~\cite{DBLP:conf/icml/RameshPGGVRCS21}, LDM~\cite{DBLP:conf/cvpr/RombachBLEO22}, Glide~\cite{DBLP:conf/icml/NicholDRSMMSC22}). We report mean Average Precision (mAP) and mean Accuracy (mAcc).

\subsubsection{Implementation Details.} 
Following~\cite{DBLP:conf/cvpr/OjhaLL23,DBLP:conf/cvpr/CuiLLZD25}, we employ the image encoder of CLIP-ViT-L/14 as our backbone. The network is optimized with Sharpness-Aware Minimization (SAM), setting the neighborhood size $\rho=0.07$, batch size to 20, and initial learning rate to $10^{-3}$. 
For the RRM (Section~\ref{subsec:rrm}), we set the coherence coefficient $\alpha=0.1$ and the number of regions $K=3$. 
For universal applicability, the three regions correspond to the foreground, foreground boundary, and background; for facial forgery tasks specifically, the foreground is the facial area.
For the HRI (Section~\ref{subsec:hri}), we set the intermediate fusion layer to $l_{\text{mid}}=18$.
All experiments are conducted on a single NVIDIA Tesla V100 GPU.

\subsection{Generalization Evaluation}
\label{sec:generalization_eval}

We evaluate the generalization of CoRIT to unseen domains through cross-dataset and cross-manipulation experiments on three benchmarks: Classical Forensic Benchmarks (Table~\ref{tab:ff_cross}, Table~\ref{tab:ff_cross_video}), Next-Generation Generative Benchmarks (Table~\ref{tab:df40_cross}), and UniversalFakeDetect Benchmarks (Table~\ref{tab:unidf_ap}, Table~\ref{tab:unidf_acc}). 
Without modifying any backbone parameters~\cite{DBLP:conf/cvpr/OjhaLL23}, CoRIT outperforms both Full Fine-Tuning (FFT) and Parameter-Efficient Fine-Tuning (PEFT) approaches.
This suggests that fully exploiting CLIP's intrinsic representation via theoretically motivated design can yield generalization highly competitive against parameter modification methods, while maintaining better stability and memory efficiency.

\subsubsection{Generalization on Classical Forensic Benchmarks}
In the frame-level evaluation on Classical Forensic Benchmarks (Table~\ref{tab:ff_cross}), CoRIT surpasses the state-of-the-art CNN-based FFT method FIA-USA across multiple unseen datasets, with a notable 9.2\% gain on DFD.
Among ViT-based approaches under the same backbone and SAM optimization settings, CoRIT outperforms its baseline UniFD, which simply employs a frozen backbone with a lightweight learnable classifier, by an average AUC margin of 12.0\%.
CoRIT also surpasses the adapter-based PEFT method ForAda by 5.1\%.
In the video-level evaluation (Table~\ref{tab:ff_cross_video}), CoRIT achieves a state-of-the-art average AUC of 91.0\%, outperforming the LoRA-based PEFT competitor Effort.

\subsubsection{Generalization to Emerging Forgeries and Universal Scenarios}
The generalization superiority of CoRIT extends to a broader spectrum of emerging facial forgery techniques and natural scene synthesis.
On the Next-Generation Generative Benchmarks (Table~\ref{tab:df40_cross}), CoRIT consistently outperforms the CNN-based FIA-USA by 4.4\% on average and improves over the frozen-backbone baseline UniFD by 11.6\%, confirming its robustness against rapidly evolving generation techniques.
On the UniversalFakeDetect Benchmarks (Table~\ref{tab:unidf_ap}, Table~\ref{tab:unidf_acc}), which span natural scenes from ImageNet and LAION, CoRIT achieves 97.17\% mAP. It surpasses the hand-crafted feature-based method NPR by 4.41\% in mAP (2.28\% in mAcc) and the frequency-aware FreqNet by 5.52\% in mAP (4.75\% in mAcc). 
Compared to UniFD, CoRIT delivers gains of 7.03\% in mAP and 8.46\% in mAcc. 
CoRIT is particularly effective on forgeries generated via perceptual loss, achieving near-perfect mAP and 97.35\% mAcc.
We attribute this to the layer-wise injection of high-GSNR features, which helps capture artifacts introduced by feature-matching optimization.
These results confirm the applicability of our approach beyond facial forgery to universal-purpose forgery detection.

\subsubsection{Cross-Manipulation Evaluation}
We further evaluate generalization to unseen manipulations on FaceForensics++ (FF++) using a leave-one-out protocol: training on three subsets and testing on the held-out one. The four manipulation types are Deepfakes (DF), Face2Face (F2F), FaceSwap (FS), and NeuralTextures (NT). 
As shown in Tab.~\ref{tab:cross_manipulation}, our CoRIT consistently outperforms both the frozen-backbone baseline UniFD and the PEFT method ForAda. The relatively lower performance on NT and F2F is expected, as the CLIP-pretrained backbone has limited sensitivity to non-semantic artifacts. Still, CoRIT improves over UniFD by 12.5\% on F2F and 8.2\% on NT, indicating that our design helps the backbone better capture manipulation traces beyond semantic-level cues.

\subsubsection{Generalization Stability}
CoRIT shows strong generalization stability during training. As shown in Fig.~\ref{fig:loss_compare}, competing methods exhibit a typical 'U-shaped' overfitting curve, whereas CoRIT maintains consistent performance under prolonged training. This indicates that our design encourages learning domain-invariant forensic representations, reducing the reliance on early stopping to avoid overfitting.

\subsubsection{Computational Efficiency}
CoRIT is highly memory-efficient during the training phase. Since the backbone remains entirely frozen, no gradient storage is needed for its parameters, reducing the overall GPU memory consumption from 21,224 MB (ForAda) to 2,130 MB, achieving a 90\% reduction.

\subsection{Theoretical Verification and Visualization} 
Guided by the theoretical analysis in Section~\ref{sec:theoretical_insights}, CoRIT consistently outperforms its baseline UniFD, as detailed in Section~\ref{sec:generalization_eval}. 
Here we provide empirical verification and visualization to support these theoretical claims.

\subsubsection{Layer-wise COR and GSNR Analysis on the NT Dataset}
We compare the layer-wise COR and Gradient Signal-to-Noise Ratio (GSNR) between UniFD and CoRIT on the NeuralTextures (NT) dataset, which is particularly challenging due to its non-semantic artifacts.
As shown in Tab.~\ref{tab:layer_cor}, UniFD exhibits a rapid COR decay, dropping below 0.05 by layer 15 and reaching 0.01 at layer 24, indicating a failure to retain generalizable non-semantic forgery patterns. CoRIT shows a distinct trend: COR values rise between layers 10 and 15, and although they decline at layer 20, the final-layer COR remains at 0.38. This ensures the COR remains well above the theoretical "Collapse Zone" (COR $< 0.05$), mitigating the Optimization Collapse.

The GSNR comparison, as illustrated in Fig.~\ref{fig:gsnr_comparison}, corroborates this finding. CoRIT maintains higher GSNR values and enters the optimization phase earlier across all layers. Following the analysis in Section~\ref{sec:theoretical_verification}, these results confirm that CoRIT strengthens the representation of non-semantic artifacts.

\subsubsection{Loss Landscape Visualization}
Fig.~\ref{fig:lossland} shows the loss landscapes of the final-layer (Layer 24) representations from UniFD and CoRIT, evaluated on a lightweight learnable classifier. 
CoRIT produces a flatter and smoother landscape, which is generally associated with better generalization~\cite{DBLP:conf/iclr/KeskarMNST17,DBLP:conf/iclr/ForetKMN21}.
The landscape for UniFD, particularly on the NT subset, is irregular with sharp curvature in multiple regions. 
As established in Lemma~\ref{lem:local_stability}, such high curvature leads to reduced COR values, consistent with the empirical results in Table~\ref{tab:layer_cor}.

\begin{table}[t]
    \centering
    \caption{Ablation study on proposed components (Frame-Level AUC (\%)). Best results are in \textbf{bold}.}
    \label{tab:ablation_study}
    
    \setlength{\tabcolsep}{4pt}
    
    \resizebox{\columnwidth}{!}{
        \begin{tabular}{ccc|cccc|c}
            \toprule
            RSI & RRM & HRI & CDFv2 & DFDC & facedancer & fsgan & Avg. \\
            \midrule
            - & - & - & 76.0 & 74.8 & 77.9 & 85.6 & 78.6 \\
            \checkmark & - & - & 86.9 & 81.6 & 79.2 & 87.3 & 83.8 \\
            \checkmark & \checkmark & - & 88.5 & 82.7 & 84.4 & 91.2 & 86.7 \\
            - & - & \checkmark & 77.4 & 79.4 & 82.6 & 89.9 & 82.3 \\
            \checkmark & \checkmark & \checkmark & \textbf{89.1} & \textbf{84.5} & \textbf{87.8} & \textbf{93.8} & \textbf{88.8} \\
            \bottomrule
        \end{tabular}
    }
\end{table}

\begin{table}[t]
    \centering
     \caption{Ablation study on spatial regions and pooling strategies (Frame-Level AUC (\%)). F, Bd, and Bg denote Foreground, Boundary, and Background, respectively. $K$ is the number of regions. Best results are in \textbf{bold}.}
    \label{tab:region_ablation}
    \setlength{\tabcolsep}{3pt} 
    \resizebox{0.48\textwidth}{!}{
        \begin{tabular}{lcc|ccccc|c}
            \toprule
            Region & $K$ & Pooling & CDFv1 & CDFv2 & DFDCP & DFDC & DFD & Avg. \\
            \midrule
            Random & 3 & Avg. & 76.1 & 79.7 & 80.9 & 78.6 & 88.0 & 80.7 \\
            F & 1 & Avg. & 88.6 & 86.6 & 84.9 & 82.5 & 91.2 & 86.8 \\
            Bd & 1 & Avg. & 76.4 & 77.3 & 78.1 & 78.2 & 87.4 & 79.5 \\
            Bg & 1 & Avg. & 67.1 & 61.7 & 70.5 & 70.0 & 77.0 & 69.3 \\
            F+Bd & 2 & Avg. & 89.3 & \textbf{89.1} & 88.5 & 83.7 & \textbf{91.3} & 88.4 \\
            F+Bg & 2 & Avg. & 89.9 & \textbf{89.1} & 87.3 & 83.2 & \textbf{91.3} & 88.2 \\
            Bd+Bg & 2 & Avg. & 77.8 & 81.2 & 79.8 & 78.0 & 88.0 & 81.0 \\
            F+Bd+Bg & 3 & Max & 83.4 & 83.6 & 84.5 & 79.7 & 89.5 & 84.1 \\
            \midrule
            F+Bd+Bg & 3 & Avg. & \textbf{90.9} & \textbf{89.1} & \textbf{88.8} & \textbf{84.5} & \textbf{91.3} & \textbf{88.9} \\
            \bottomrule
        \end{tabular}
    }
\end{table}

\begin{table}[t]
    \centering
    \caption{Ablation study across different backbones: IMG-L (ImageNet-ViT-Large), CLIP-B (CLIP-ViT-Base), and CLIP-L (CLIP-ViT-Large). Best results are in \textbf{bold}.}
    \label{tab:backbone_ablation}

    \setlength{\tabcolsep}{3.5pt}
    
    \resizebox{0.95\columnwidth}{!}{
        \begin{tabular}{l|l|cccc|c}
            \toprule
            Backbone & Method & CDFv1 & CDFv2 & DFDCP & DFDC & Avg. \\
            \midrule

            \multirow{2}{*}{IMG-L} 
                & UniFD~\cite{DBLP:conf/cvpr/OjhaLL23} & 53.0 & 54.0 & 63.2 & 61.4 & 57.9 \\
                & Ours & 66.3 & 65.0 & 75.0 & 67.6 & 68.5 \\
            \midrule
            \multirow{2}{*}{CLIP-B} 
                & UniFD~\cite{DBLP:conf/cvpr/OjhaLL23} & 75.0 & 73.3 & 72.4 & 70.0 & 72.7 \\
                & Ours & 81.0 & 77.2 & 80.8 & 76.5 & 78.9 \\
            \midrule
            \multirow{2}{*}{CLIP-L} 
                & UniFD~\cite{DBLP:conf/cvpr/OjhaLL23} & 71.2 & 76.0 & 73.6 & 74.8 & 73.9 \\
                & \textbf{Ours} & \textbf{90.9} & \textbf{89.1} & \textbf{88.8} & \textbf{84.5} & \textbf{88.3} \\
            \bottomrule
        \end{tabular}
    }
\end{table}

\begin{table}[t]
    \centering
    \caption{Sensitivity analysis of the coherence coefficient $\alpha$ on classical forensic benchmarks (Frame-Level AUC (\%)). Best results are in \textbf{bold}.}
    \label{tab:ablation_alpha}
    
    \setlength{\tabcolsep}{3pt} 
    \resizebox{0.4\textwidth}{!}{
        \begin{tabular}{c|ccccc|c}
            \toprule
            $\alpha$ & CDFv1 & CDFv2 & DFDCP & DFDC & DFD & Avg. \\
            \midrule
            0 & 88.6 & 88.5 & 89.1 & 84.1 & 91.2 & 88.3 \\

            0.1 & \textbf{90.9} & \textbf{89.1} & 88.8 & \textbf{84.5} & \textbf{91.3} & \textbf{88.9} \\
            0.5 & 89.1 & 88.7 & \textbf{89.3} & 84.3 & \textbf{91.3} & 88.5 \\
            0.9 & 88.9 & 88.7 & 89.1 & 84.2 & 91.1 & 88.4 \\
            1.3 & 87.9 & 86.4 & 85.3 & 82.3 & 90.1 & 86.4 \\
            \bottomrule
        \end{tabular}%
    }
\end{table}

\begin{table}[t]
    \centering
    \caption{Ablation study on the selection of intermediate layers for feature fusion. The index indicates the encoder layer used for extracting shallow features. Best results are in \textbf{bold}.}
    \label{tab:ablation_layers}

    \setlength{\tabcolsep}{3.5pt}
    
    \resizebox{0.8\columnwidth}{!}{
        \begin{tabular}{c|ccccc|c}
            \toprule
            Layer & CDFv1 & CDFv2 & DFDCP & DFDC & DFD & Avg. \\
            \midrule
            10 & 86.6 & 86.8 & 85.3 & 84.1 & 91.5 & 86.9 \\
            16 & 88.4 & 88.2 & 87.8 & 84.4 & \textbf{91.9} & 88.1 \\
            18 & \textbf{90.9} & \textbf{89.1} & \textbf{88.4} & 84.5 & 91.5 & \textbf{88.9} \\
            20 & 90.3 & 88.9 & 85.5 & \textbf{85.0} & 91.2 & 88.2 \\
            \bottomrule
        \end{tabular}
    }
\end{table}

\begin{table}[t]
    \centering
    \caption{Ablation study on fusion architecture. We compare feature concatenation (Concat) with a cross-attention mechanism (CrossAttn). Best results are in \textbf{bold}.}
    \label{tab:classifier_ablation}

    \setlength{\tabcolsep}{3.5pt}
    
    \resizebox{0.9\columnwidth}{!}{
        \begin{tabular}{l|ccccc|c}
            \toprule
            Type & CDFv1 & CDFv2 & DFDCP & DFDC & DFD & Avg. \\
            \midrule
            CrossAttn & \textbf{91.5} & 88.4 & 85.6 & 83.9 & 90.4 & 88.0 \\
            Concat & 90.9 & \textbf{89.1} & \textbf{88.4} & \textbf{84.5} & \textbf{91.5} & \textbf{88.9} \\
            \bottomrule
        \end{tabular}
    }
\end{table}

\subsection{Ablation Study}
\label{subsec:ablation}

We conduct a comprehensive ablation study to verify the effectiveness of each component.

\subsubsection{Effectiveness of Proposed Components} 
As shown in Table~\ref{tab:ablation_study}, we progressively validate the contribution of Regional Signal Injection (RSI), Region Refinement Mask (RRM), and Hierarchical Representation Integration (HRI). 
Compared to the baseline, the introduction of RSI yields a substantial initial gain of 5.2\%.
RRM brings an additional 2.9\% by filtering incoherent gradients. 
HRI contributes a further 2.1\%, resulting in a total improvement of 10.2\% over the baseline. 
When applied alone, HRI also improves the baseline by 3.7\%, verifying its independent efficacy. 
Each component thus plays a distinct and complementary role in improving forensic generalization.

\subsubsection{Ablation on Spatial Regions and Pooling} 
Table~\ref{tab:region_ablation} presents the ablation on region configurations, including the number of regions $K$ and pooling strategies. 
With a single region ($K=1$), the foreground captures the most discriminative cues at 86.8\% average AUC, which significantly outperforms random sampling. 
Two-region combinations ($K=2$), such as foreground paired with boundary, raise performance to 88.4\%. 
Using all three regions ($K=3$) brings further gains on CDFv1 and DFDC, yielding a total improvement of 2.1\% over foreground alone, which confirms the complementary value of boundary and background features. 
For pooling strategies, average pooling outperforms max pooling by 4.8\% under the full-context setting, suggesting it better suppresses domain-specific noise that hinders generalization.

\begin{figure*}[t]
    \centering
    \includegraphics[width=1\textwidth]{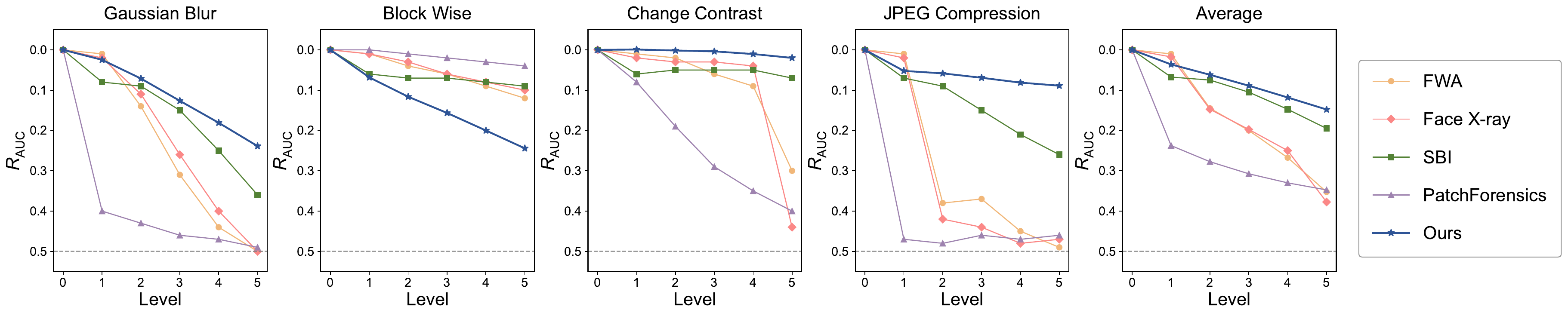}
    \caption{Robustness evaluation under different perturbation levels: Gaussian Blur, Block-wise Distortion, Contrast Change, and JPEG Compression. Relative AUC Variation curves show AUC degradation as perturbation severity increases.}
    \label{fig:robustness}
\end{figure*}

\subsection{Sensitivity Analysis} 

\subsubsection{Impact of Backbones} Table~\ref{tab:backbone_ablation} evaluates CoRIT across backbone scales and pre-training paradigms, including ImageNet~\cite{deng2009imagenet} and CLIP. CoRIT consistently outperforms UniFD under all backbones, with a 6.2\% gain on CLIP-pretrained ViT-Base and 10.6\% on ImageNet-pretrained ViT-L. 
These results show that our framework is robust to model scale and extends beyond multi-modal architectures to benefit backbones pre-trained for semantic classification. We adopt CLIP-ViT-L as the default backbone given its best overall performance.

\subsubsection{Impact of Mask Threshold}
Tab.~\ref{tab:ablation_alpha} shows that performance remains stable across the range $\alpha \in [0, 0.9]$, demonstrating the model's robustness to this hyperparameter. We adopt $\alpha=0.1$ for all experiments.

\subsubsection{Impact of Fusion Layers}
Tab.~\ref{tab:ablation_layers} shows that performance remains stable across the range $l_{mid} \in [16, 20]$, indicating low sensitivity to the choice of fusion layer. We set $l_{mid}=18$ for all experiments.

\subsubsection{Impact of Fusion Strategy} 
We evaluate the sensitivity of feature fusion in Tab.~\ref{tab:classifier_ablation} by comparing the Concatenation ($\mathbf{F}_{\text{cat}} = [\mathbf{F}^{(l_{\text{mid}})} \parallel \mathbf{F}^{(L)}]$) with a Cross-Attention mechanism:
\begin{equation}
    \mathbf{F}_{\text{attn}} = \mathbf{F}^{(l_{\text{mid}})} + \text{CrossAttn}(\mathbf{F}^{(L)}, \mathbf{F}^{(l_{\text{mid}})}, \mathbf{F}^{(l_{\text{mid}})}),
\end{equation}
where $\mathbf{F}^{(L)}$ serves as the query to extract details from $\mathbf{F}^{(l_{\text{mid}})}$. The two strategies differ by only 0.9\%, confirming that our method is insensitive to classifier design. We therefore adopt the more efficient Concatenation for all experiments.

\subsection{Robustness Analysis}
To comprehensively assess the stability of our detector against common image degradations, we introduce four types of random perturbations: Gaussian Blur, Block-wise Distortion, Contrast Change, and JPEG Compression. We quantify the impact of these disturbances using the Relative AUC Variation ($R_{\text{AUC}}$), defined as:
\begin{equation}
    R_{\text{AUC}} = \frac{\text{AUC}_{\text{ptb}} - \text{AUC}_{\text{raw}}}{\text{AUC}_{\text{raw}}},
\end{equation}
where $\text{AUC}_{\text{raw}}$ and $\text{AUC}_{\text{ptb}}$ denote performance on the original and perturbed datasets, respectively. $R_{\text{AUC}}$ Values closer to zero indicate higher robustness. 

As shown in Fig.~\ref{fig:robustness}, all methods degrade as perturbation severity increases, but CoRIT exhibits the slowest decay rate on average, particularly under heavy JPEG Compression (levels 1--5), which confirms the effectiveness of our noise filtering mechanism. One exception is Block-wise Distortion, where CoRIT shows higher sensitivity. We attribute this to CLIP's intrinsic prior on semantic integrity. 
Unlike local noise, block-wise occlusion disrupts the semantic integrity of the image.
Masking key regions introduces ambiguity that itself resembles a form of manipulation. We therefore consider this sensitivity reasonable behavior for deepfake detectors.

\section{Conclusion and Discussion}

\subsection{Conclusion}
This work reveals that Optimization Collapse in CLIP-based forensics is not merely a geometric instability induced by gradient perturbations, but a direct manifestation of degraded GSNR arising from insufficient intrinsic generalizability. 
This degradation is inherent to semantic-centric backbones, whose representations systematically attenuate the non-semantic forensic artifacts essential to detecting hyper-realistic forgeries.
The COR-GSNR Stability Decomposition theorem formalizes this connection, establishing that the vanishing of COR implies the degradation of GSNR, thereby bridging optimization geometry with generalization potential.
Built on this theoretical foundation, CoRIT actively counteracts the layer-wise attenuation of non-semantic signals and achieves state-of-the-art generalization across comprehensive benchmarks.
These results demonstrate that enhancing gradient signal fidelity can mitigate Optimization Collapse and unlock the forensic potential of frozen semantic-centric backbones.

\subsection{Discussion}
Our theoretical findings extend beyond the proposed method, offering a quantitative lens for evaluating forensic benchmarks themselves.
As shown in Fig.~\ref{fig:oc_visual}(a), current benchmarks are predominantly populated by high-COR forgery types such as Celeb-DF and uniface, where semantic artifacts remain readily exploitable.
In contrast, low-COR forgeries like NeuralTextures and Face2Face, whose artifacts reside almost entirely in the non-semantic domain, pose substantially greater challenges yet remain underrepresented.
This imbalance suggests that COR and GSNR can serve as principled, representation-aware metrics to guide future benchmark construction.
More broadly, as generative models continue to advance toward photorealistic synthesis, the proportion of low-COR forgeries will inevitably grow.
Prioritizing such samples in benchmark design would more rigorously stress-test next-generation detectors, ensuring that progress in forensic generalization keeps pace with the relentless evolution of synthesis fidelity.

\bibliographystyle{IEEEtran}
\bibliography{main_references}

\newpage
\section{Theoretical Derivations and Supporting Analysis}
\label{supplementary:supplementary_A}

This supplementary collects the detailed mathematical derivations and supporting analyses underlying the theoretical framework presented in the main text. The material is organized as follows.

\begin{itemize}
    \item Section~\ref{subsec:hessian_decomposition} derives the Hessian-Covariance decomposition, establishing the structural decomposition of the expected Hessian matrix $\mathbf{H}_t$ into the gradient covariance $\operatorname{Cov}[\mathbf{g}_t]$, the outer product of the expected gradient, and the model misspecification residual $\boldsymbol{\Xi}_t$.
    \item Section~\ref{sec:xi} proves that the COR-GSNR Stability Decomposition in Theorem~1 of the main text is well-posed throughout the entire optimization trajectory, and discusses the asymptotic vanishing of $\boldsymbol{\Xi}_t$ near convergence.
    \item Section~\ref{app:justification_curvature} provides a rigorous justification for approximating the local smoothness constant $L_t$ by $\lambda_{\max}(\mathbf{H}_t)$, drawing on recent theoretical results for Sharpness-Aware Minimization.
    \item Section~\ref{sec:cor_monotonicity} establishes the strict monotonic dependence of $\rho_{\mathrm{critical}}$ on $\mathrm{GSNR}_{t^*}$.
\end{itemize}

\vspace{0.5em}
\noindent For ease of reference, Table~\ref{tab:notation} summarises the principal symbols used throughout this supplementary.

\begin{table}[h]
\centering
\caption{Summary of principal notation.}
\label{tab:notation}
\renewcommand{\arraystretch}{1.25}
\begin{tabular}{@{}c p{7.2cm}@{}}
\toprule
Symbol & Definition \\
\midrule
\multicolumn{2}{@{}l}{\textit{Optimization \& Loss}} \\[2pt]
$\mathbf{w}_t$ & Model parameters at optimization step $t$ \\
$\mathcal{L}_t$ & Expected NLL loss $\mathbb{E}_{\mathbf{x}\sim q}[-\log p_{\mathbf{w}_t}(\mathbf{x})]$ \\
$\mathbf{g}(\mathbf{x},\mathbf{w}_t)$ & Stochastic gradient for a single sample $\mathbf{x}$ \\
$\nabla\mathcal{L}_t$ & Expected (population) gradient \\
$\mathbf{H}_t$ & Expected Hessian matrix $\nabla^2\mathcal{L}(\mathbf{w}_t)$ \\
$L_t$ & Local smoothness constant within the $\rho$-neighborhood \\
$t^*$ & Bottleneck optimization step \\
\midrule
\multicolumn{2}{@{}l}{\textit{Hessian Decomposition}} \\[2pt]
$\operatorname{Cov}[\mathbf{g}_t]$ & Gradient covariance matrix \\
$\boldsymbol{\Xi}_t$ & Model misspecification residual $\mathbb{E}_{q}\!\bigl[\nabla^2 p_{\mathbf{w}_t}/p_{\mathbf{w}_t}\bigr]$ \\
$\kappa_s(\mathbf{H}_t)$ & Stable rank $\operatorname{Tr}(\mathbf{H}_t)/\lambda_{\max}(\mathbf{H}_t)$ \\
$\lambda_{\max}(\mathbf{H}_t)$ & Largest eigenvalue of $\mathbf{H}_t$ \\
\midrule
\multicolumn{2}{@{}l}{\textit{Core Indicators}} \\[2pt]
$\mathrm{GSNR}_{t^*}$ & Gradient Signal-to-Noise Ratio at $t^*$: $\|\nabla\mathcal{L}_{t^*}\|^2 / \operatorname{Tr}(\operatorname{Cov}[\mathbf{g}_{t^*}])$ \\
$\rho_{\mathrm{critical}}$ & Critical Optimization Radius \\
\midrule
\multicolumn{2}{@{}l}{\textit{Distributions}} \\[2pt]
$q(\mathbf{x})$ & True data distribution \\
$p_{\mathbf{w}_t}(\mathbf{x})$ & Parameterised model distribution at step $t$ \\
\bottomrule
\end{tabular}
\end{table}

\subsection{Decomposition of the Hessian Matrix}
\label{subsec:hessian_decomposition}

Let $\mathbf{x}$ be the input data drawn from the true data distribution $q(\mathbf{x})$, and let $p_{\mathbf{w}_t}(\mathbf{x})$ denote the parameterized model distribution at optimization step $t$. Under the standard classification setting, the objective function $\mathcal{L}(\mathbf{w}_t)$ is the expected Negative Log-Likelihood (NLL) loss:
\begin{equation}
    \mathcal{L}_t \triangleq \mathbb{E}_{\mathbf{x} \sim q(\mathbf{x})}[-\log p_{\mathbf{w}_t}(\mathbf{x})].
\end{equation}

The stochastic gradient for a single sample $\mathbf{x}$, denoted by $\mathbf{g}(\mathbf{x}, \mathbf{w}_t)$, is computed as:
\begin{equation}
    \mathbf{g}(\mathbf{x}, \mathbf{w}_t) = -\nabla_{\mathbf{w}_t} \log p_{\mathbf{w}_t}(\mathbf{x}) = -\frac{\nabla_{\mathbf{w}_t} p_{\mathbf{w}_t}(\mathbf{x})}{p_{\mathbf{w}_t}(\mathbf{x})}.
    \label{eq:app_grad}
\end{equation}
The expected global gradient is therefore $\nabla \mathcal{L}_t = \mathbb{E}_{\mathbf{x} \sim q(\mathbf{x})}[\mathbf{g}(\mathbf{x}, \mathbf{w}_t)]$.

Next, we compute the Hessian matrix for a single sample by differentiating the gradient in \eqref{eq:app_grad} with respect to $\mathbf{w}_t$. Applying the quotient rule yields:
\begin{equation}
\begin{aligned}
    \mathbf{H}(\mathbf{x}, \mathbf{w}_t) 
    &= -\nabla_{\mathbf{w}_t}^2 \log p_{\mathbf{w}_t}(\mathbf{x}) \\
    &= -\nabla_{\mathbf{w}_t} \left( \frac{\nabla_{\mathbf{w}_t} p_{\mathbf{w}_t}(\mathbf{x})}{p_{\mathbf{w}_t}(\mathbf{x})} \right) \\
    &= -\frac{\nabla_{\mathbf{w}_t}^2 p_{\mathbf{w}_t}(\mathbf{x}) \cdot p_{\mathbf{w}_t}(\mathbf{x}) - \nabla_{\mathbf{w}_t} p_{\mathbf{w}_t}(\mathbf{x}) \nabla_{\mathbf{w}_t} p_{\mathbf{w}_t}(\mathbf{x})^\top}{(p_{\mathbf{w}_t}(\mathbf{x}))^2} \\
    &= \left( \frac{\nabla_{\mathbf{w}_t} p_{\mathbf{w}_t}(\mathbf{x})}{p_{\mathbf{w}_t}(\mathbf{x})} \right) \left( \frac{\nabla_{\mathbf{w}_t} p_{\mathbf{w}_t}(\mathbf{x})}{p_{\mathbf{w}_t}(\mathbf{x})} \right)^\top - \frac{\nabla_{\mathbf{w}_t}^2 p_{\mathbf{w}_t}(\mathbf{x})}{p_{\mathbf{w}_t}(\mathbf{x})} \\
    &= \mathbf{g}(\mathbf{x}, \mathbf{w}_t) \mathbf{g}(\mathbf{x}, \mathbf{w}_t)^\top - \frac{\nabla_{\mathbf{w}_t}^2 p_{\mathbf{w}_t}(\mathbf{x})}{p_{\mathbf{w}_t}(\mathbf{x})}.
\end{aligned}
\end{equation}

Taking the expectation over the true data distribution $q(\mathbf{x})$, the expected Hessian matrix $\mathbf{H}_t$ is formulated as:
\begin{equation}
    \mathbf{H}_t = \mathbb{E}_{\mathbf{x} \sim q(\mathbf{x})} [\mathbf{g}(\mathbf{x}, \mathbf{w}_t) \mathbf{g}(\mathbf{x}, \mathbf{w}_t)^\top] - \mathbb{E}_{\mathbf{x} \sim q(\mathbf{x})} \left[ \frac{\nabla_{\mathbf{w}_t}^2 p_{\mathbf{w}_t}(\mathbf{x})}{p_{\mathbf{w}_t}(\mathbf{x})} \right].
    \label{eq:app_expected_hessian}
\end{equation}

Recall that the uncentered second moment of the gradient relates to the gradient covariance matrix $\operatorname{Cov}[\mathbf{g}_t]$ and the expected gradient $\nabla \mathcal{L}_t$ via the standard identity:
\begin{equation}
    \mathbb{E}_{\mathbf{x} \sim q(\mathbf{x})} [\mathbf{g}(\mathbf{x}, \mathbf{w}_t) \mathbf{g}(\mathbf{x}, \mathbf{w}_t)^\top] = \operatorname{Cov}[\mathbf{g}_t] + \nabla \mathcal{L}_t \nabla \mathcal{L}_t^\top.
    \label{eq:app_cov}
\end{equation}

Substituting \eqref{eq:app_cov} into \eqref{eq:app_expected_hessian}, we obtain the precise structural decomposition of the Hessian:
\begin{equation}
    \mathbf{H}_t = \operatorname{Cov}[\mathbf{g}_t] + \nabla \mathcal{L}_t \nabla \mathcal{L}_t^\top - \boldsymbol{\Xi}_t,
\end{equation}
where $\boldsymbol{\Xi}_t \triangleq \mathbb{E}_{\mathbf{x} \sim q(\mathbf{x})} \left[ \frac{\nabla_{\mathbf{w}_t}^2 p_{\mathbf{w}_t}(\mathbf{x})}{p_{\mathbf{w}_t}(\mathbf{x})} \right]$ defines the residual matrix capturing the model misspecification.

\subsection{Well-Posedness of the Residual Matrix $\boldsymbol{\Xi}_t$}
\label{sec:xi}

We establish that the COR formula in Theorem~1 of the main text is always real-valued, regardless of the sign of $\operatorname{Tr}(\boldsymbol{\Xi}_{t^*})$.

\begin{lemma}[Well-Posedness of the Misspecification Term]
\label{lem:xi_wellposed}
Under the cross-entropy loss with $\operatorname{Tr}(\mathbf{H}_{t^*}) > 0$, the misspecification term in Theorem~1 of the main text satisfies:
\begin{equation}
    1 + \frac{\operatorname{Tr}(\boldsymbol{\Xi}_{t^*})}{\operatorname{Tr}(\mathbf{H}_{t^*})} \geq 0,
\end{equation}
thereby guaranteeing that the square root in~(13) of the main text is real-valued.
\end{lemma}

\begin{proof}
Taking the trace of the Hessian decomposition~(5) of the main text yields:
\begin{equation}
    \operatorname{Tr}(\operatorname{Cov}[\mathbf{g}_{t^*}]) 
    = \operatorname{Tr}(\mathbf{H}_{t^*}) + \operatorname{Tr}(\boldsymbol{\Xi}_{t^*}) 
      - \|\nabla\mathcal{L}_{t^*}\|^2.
\end{equation}
Since $\operatorname{Cov}[\mathbf{g}_{t^*}]$ is positive semi-definite, $\operatorname{Tr}(\operatorname{Cov}[\mathbf{g}_{t^*}]) \geq 0$. Rearranging directly yields:
\begin{equation}
    \operatorname{Tr}(\mathbf{H}_{t^*}) + \operatorname{Tr}(\boldsymbol{\Xi}_{t^*}) 
    \geq \|\nabla\mathcal{L}_{t^*}\|^2 \geq 0.
\end{equation}
Dividing both sides by $\operatorname{Tr}(\mathbf{H}_{t^*}) > 0$ gives the desired result. This guarantee follows purely from the positive semi-definiteness of the covariance matrix, without any assumption on the sign of $\operatorname{Tr}(\boldsymbol{\Xi}_{t^*})$ itself.
\end{proof}

\begin{remark}[Behavior across optimization regimes]
\label{rem:xi_regimes}
Under standard regularity conditions, $\boldsymbol{\Xi}_t \to \mathbf{0}$ as $p_{\mathbf{w}_t} \to q$, since $\operatorname{Tr}(\boldsymbol{\Xi}_t) = \int q(\mathbf{x}) \, \nabla_{\mathbf{w}_t}^2 p_{\mathbf{w}_t}(\mathbf{x}) / p_{\mathbf{w}_t}(\mathbf{x}) \, d\mathbf{x}$ reduces to $\nabla_{\mathbf{w}_t}^2 \!\int p_{\mathbf{w}_t}(\mathbf{x}) \, d\mathbf{x} = \mathbf{0}$ when the model perfectly fits the data. This asymptotic vanishing is precisely \emph{not} the regime of interest: Optimization Collapse occurs during early optimization (Phase~(i) in Section~III-F of the main text), where $p_{\mathbf{w}_t}$ diverges significantly from $q$ and $\boldsymbol{\Xi}_t$ remains non-negligible. Crucially, however, Lemma~\ref{lem:xi_wellposed} ensures the well-posedness of the COR decomposition throughout the entire optimization trajectory, including this pre-convergence regime.
\end{remark}

\subsection{Theoretical Justification for Local Curvature Assumption}
\label{app:justification_curvature}

In our theoretical framework, the local smoothness constant $L_t$ within the $\rho$-neighborhood is bounded using the maximum eigenvalue of the Hessian at the center, i.e., $L_t \approx \lambda_{\max}(\mathbf{H}_t)$. This formulation is rigorously grounded in recent theoretical treatments of Sharpness-Aware Minimization (SAM).

Following Long and Bartlett~\cite{DBLP:journals/jmlr/LongB24}, the loss landscape within a local neighborhood can be characterized by a local quadratic approximation. For any weight vector $\mathbf{w}$ near $\mathbf{w}_t$, the loss is approximated as:
\begin{equation}
    \mathcal{L}(\mathbf{w}) \approx \mathcal{L}(\mathbf{w}_t) + \nabla \mathcal{L}(\mathbf{w}_t)^\top(\mathbf{w} - \mathbf{w}_t) + \frac{1}{2}(\mathbf{w} - \mathbf{w}_t)^\top \mathbf{H}_t (\mathbf{w} - \mathbf{w}_t),
\end{equation}
where $\mathbf{H}_t = \nabla^2 \mathcal{L}(\mathbf{w}_t)$. The local optimization dynamics and the bounds for the edge of stability in SAM are fundamentally governed by the operator norm of the Hessian, $\|\mathbf{H}_t\|_{\mathrm{op}}$~\cite{DBLP:journals/jmlr/LongB24}. 
Since the operator norm equals the largest eigenvalue in absolute value, i.e., $\|\mathbf{H}_t\|_{\mathrm{op}} = \max\{|\lambda_{\max}(\mathbf{H}_t)|, |\lambda_{\min}(\mathbf{H}_t)|\}$, this approximation holds when the magnitude of the most negative eigenvalue is dominated by $\lambda_{\max}(\mathbf{H}_t)$, a condition broadly observed in practice after the initial training phase~\cite{sagun2017empirical,ghorbani2019investigation}.
\begin{equation}
    \|\mathbf{H}_t\|_{\mathrm{op}} \approx \lambda_{\max}(\mathbf{H}_t).
\end{equation}

To mathematically justify extending this pointwise metric to the supremum over the entire neighborhood $\mathcal{N}(\mathbf{w}_t, \rho)$, we rely on the framework established by Bartlett et al.~\cite{DBLP:journals/jmlr/BartlettLB23}. By assuming that the third-order derivative tensor $D^3 \mathcal{L}$ is $B$-Lipschitz continuous with respect to the operator norm, the variation of the Hessian along any perturbation vector $\boldsymbol{\epsilon}$ (with $\|\boldsymbol{\epsilon}\|_2 \le \rho$) can be evaluated via the fundamental theorem of calculus:
\begin{equation}
    \nabla^2 \mathcal{L}(\mathbf{w}_t + \boldsymbol{\epsilon}) = \mathbf{H}_t + \int_{0}^{1} D^3 \mathcal{L}(\mathbf{w}_t + x\boldsymbol{\epsilon})(\boldsymbol{\epsilon}, \cdot, \cdot) \, dx.
\end{equation}
Applying the $B$-Lipschitz property bounds the spectral norm for any perturbed weight along this path~\cite{DBLP:journals/jmlr/BartlettLB23}:
\begin{equation}
    \|\nabla^2 \mathcal{L}(\mathbf{w}_t + \boldsymbol{\epsilon})\|_{\mathrm{op}} \le \lambda_{\max}(\mathbf{H}_t) + \mathcal{O}(\rho).
\end{equation}
Under the Optimization Collapse phenomenon, where $\rho$ is sufficiently small, the error term $\mathcal{O}(\rho)$ becomes vanishingly small. Therefore, taking the supremum over the neighborhood yields:
\begin{equation}
    L_t = \sup_{\|\boldsymbol{\epsilon}\|_2 \le \rho} \|\nabla^2 \mathcal{L}(\mathbf{w}_t + \boldsymbol{\epsilon})\|_{\mathrm{op}} \approx \lambda_{\max}(\mathbf{H}_t).
\end{equation}
This derivation guarantees that $\lambda_{\max}(\mathbf{H}_t)$ serves as a theoretically sound and tight proxy for the local smoothness constant $L_t$ of the population loss. The extension to mini-batch gradients follows naturally: if each per-sample loss $\ell_i(\mathbf{w})$ satisfies $L_t$-smoothness within $\mathcal{N}(\mathbf{w}_t, \rho)$, then the mini-batch loss $\mathcal{L}_{\mathcal{B}_t} = \frac{1}{|\mathcal{B}_t|}\sum_{i \in \mathcal{B}_t} \ell_i$ inherits the same constant, since for any $\|\boldsymbol{\epsilon}\| \le \rho$:
\begin{align}
    &\|\nabla \mathcal{L}_{\mathcal{B}_t}(\mathbf{w}_t + \boldsymbol{\epsilon}) - \nabla \mathcal{L}_{\mathcal{B}_t}(\mathbf{w}_t)\| \notag \\
    &\quad \le \frac{1}{|\mathcal{B}_t|}\sum_{i \in \mathcal{B}_t} \|\nabla \ell_i(\mathbf{w}_t + \boldsymbol{\epsilon}) - \nabla \ell_i(\mathbf{w}_t)\| \notag \\
    &\quad \le L_t \|\boldsymbol{\epsilon}\|.
\end{align}
This per-sample smoothness assumption is standard in SAM analysis~\cite{DBLP:conf/iclr/ForetKMN21, DBLP:journals/jmlr/BartlettLB23, DBLP:journals/jmlr/LongB24}.

\subsection{Monotonic Dependence of the Critical Optimization Radius on the Gradient Signal-to-Noise Ratio}
\label{sec:cor_monotonicity}

The following corollary derives two quantitative consequences of Theorem~1 of the main text that directly motivate the design principle in Section~III-G. We first introduce a regularity assumption on the non-statistical components of the decomposition.

\begin{assumption}[Regularity of Non-Statistical Terms]
\label{assum:regularity}
At the bottleneck step $t^*$, the geometric term and the misspecification term in Theorem~1 are bounded away from zero. That is, there exist constants $c_g, c_m > 0$ such that:
\begin{equation}
    \frac{\kappa_s(\mathbf{H}_{t^*})}{\sqrt{\operatorname{Tr}(\mathbf{H}_{t^*})}} 
    \geq c_g
    \quad \text{and} \quad
    \sqrt{1 + \frac{\operatorname{Tr}(\boldsymbol{\Xi}_{t^*})}
                   {\operatorname{Tr}(\mathbf{H}_{t^*})}} 
    \geq c_m.
    \label{eq:regularity}
\end{equation}
\end{assumption}

\noindent\textbf{Justification.} 
Lemma~\ref{lem:xi_wellposed} establishes the non-negativity $1 + \operatorname{Tr}(\boldsymbol{\Xi}_{t^*})/\operatorname{Tr}(\mathbf{H}_{t^*}) \geq 0$. The strict positivity $c_m > 0$ additionally excludes the degenerate case $\operatorname{Tr}(\operatorname{Cov}[\mathbf{g}_{t^*}]) = \|\nabla\mathcal{L}_{t^*}\|^2$, which would imply zero gradient variance in all but one direction and is not observed in practice. The lower bound $c_g > 0$ requires only that the loss landscape is not identically flat, i.e., $\kappa_s(\mathbf{H}_{t^*}) > 0$ and $\operatorname{Tr}(\mathbf{H}_{t^*}) < \infty$, both of which hold throughout the active optimization phase.

\begin{corollary}[COR Monotonicity and Collapse Equivalence]
\label{cor:monotonicity}
Under Assumption~\ref{assum:regularity}, the following two statements hold.

\noindent(i) Strict Monotonicity.
At a given bottleneck step $t^*$, with the geometric and misspecification terms held fixed, the Critical Optimization Radius $\rho_{\mathrm{critical}}$ is a strictly monotonically increasing function of $\mathrm{GSNR}_{t^*}$:
\begin{equation}
    \frac{\partial\, \rho_{\mathrm{critical}}}
         {\partial\, \mathrm{GSNR}_{t^*}} > 0 
    \qquad \forall\; \mathrm{GSNR}_{t^*} > 0.
    \label{eq:monotonicity}
\end{equation}

\noindent(ii) Collapse Equivalence.
The following biconditional equivalence holds:
\begin{equation}
    \rho_{\mathrm{critical}} \to 0 
    \quad \Longleftrightarrow \quad 
    \mathrm{GSNR}_{t^*} \to 0.
    \label{eq:equivalence}
\end{equation}
Moreover, in the regime $\mathrm{GSNR}_{t^*} \ll 1$, the COR satisfies the first-order approximation:
\begin{equation}
    \rho_{\mathrm{critical}} 
    \;\approx\; C \cdot \sqrt{\mathrm{GSNR}_{t^*}},
    \label{eq:collapse_approx}
\end{equation}
where $C > 0$ is a positive constant determined by the geometric and misspecification terms at $t^*$.
\end{corollary}

\begin{proof}
Let $s \triangleq \mathrm{GSNR}_{t^*} \geq 0$ and define the scalar function $f : [0, \infty) \to [0,1)$ by:
\begin{equation}
    f(s) \;\triangleq\; \sqrt{\frac{s}{1+s}}.
\end{equation}
By Theorem~1 and Assumption~\ref{assum:regularity}, there exist constants $c_g' \geq c_g > 0$ and $c_m' \geq c_m > 0$ denoting the actual values of the geometric and misspecification terms, respectively, such that:
\begin{equation}
    \rho_{\mathrm{critical}} = c_g' \cdot c_m' \cdot f(s).
    \label{eq:rho_factored}
\end{equation}

Proof of (i).
Since Assumption~\ref{assum:regularity} treats the geometric and misspecification terms as fixed at the bottleneck step $t^*$, the monotonicity is understood as holding pointwise, i.e., $\rho_{\mathrm{critical}}$ increases whenever $\mathrm{GSNR}_{t^*}$ increases while the remaining terms are held constant. Differentiating $f$ with respect to $s$ via the chain rule yields:
\begin{equation}
\begin{aligned}
    f'(s) 
    &= \frac{d}{ds}\left(\frac{s}{1+s}\right)^{1/2}
     = \frac{1}{2} \left(\frac{s}{1+s}\right)^{-1/2} 
       \cdot \frac{1}{(1+s)^2} \\
    &= \frac{1}{2(1+s)^2} \cdot \sqrt{\frac{1+s}{s}}
     = \frac{1}{2(1+s)\sqrt{s(1+s)}} \;>\; 0.
\end{aligned}
\end{equation}
Since $c_g' > 0$, $c_m' > 0$, and $f'(s) > 0$ for all $s > 0$, it follows from~\eqref{eq:rho_factored} that:
\begin{equation}
    \frac{\partial\, \rho_{\mathrm{critical}}}
         {\partial\, \mathrm{GSNR}_{t^*}} 
    = c_g' \cdot c_m' \cdot f'(s) > 0,
\end{equation}
which establishes~\eqref{eq:monotonicity}.

Proof of (ii).
Define $C \triangleq c_g' \cdot c_m' > 0$. Since $f$ is continuous on $[0,\infty)$ and satisfies $f(s) = 0 \iff s = 0$, $f(s) > 0 \iff s > 0$, it follows from~\eqref{eq:rho_factored} that:
\begin{equation}
    \rho_{\mathrm{critical}} \to 0 
    \;\iff\; f(s) \to 0 
    \;\iff\; s \to 0,
\end{equation}
which establishes the biconditional~\eqref{eq:equivalence}. To characterize the scaling rate near collapse, we expand $(1+s)^{-1/2}$ for $s \ll 1$ via Taylor's theorem:
\begin{equation}
    f(s) 
    = \sqrt{s} \cdot (1+s)^{-1/2}
    = \sqrt{s} \left(1 - \frac{s}{2} + \mathcal{O}(s^2)\right)
    = \sqrt{s} + \mathcal{O}(s^{3/2}).
\end{equation}
Substituting into~\eqref{eq:rho_factored} yields~\eqref{eq:collapse_approx}:
\begin{equation}
    \rho_{\mathrm{critical}} = C \cdot \left(\sqrt{s} + \mathcal{O}(s^{3/2})\right) \approx C \cdot \sqrt{\mathrm{GSNR}_{t^*}}. \qquad \qedhere
\end{equation}
\end{proof}

\begin{remark}
Corollary~\ref{cor:monotonicity} formalizes the design principle of Section~III-G in the main text. Statement~(i) provides the foundational justification: at a given bottleneck step, any increase in $\mathrm{GSNR}_{t^*}$, however marginal, strictly expands the stability boundary. Statement~(ii) further characterizes the collapse dynamics: since $\rho_{\mathrm{critical}} = \mathcal{O}(\sqrt{\mathrm{GSNR}_{t^*}})$ in the critical regime, the COR is acutely sensitive to small perturbations in gradient quality near zero, providing a theoretical account of the abruptness of the collapse observed empirically in Fig.~2. Taken together, \eqref{eq:monotonicity}--\eqref{eq:equivalence} elevate $\mathrm{GSNR}_{t^*}$ from an empirical correlate to a \emph{necessary and sufficient indicator} of Optimization Collapse, under the mild regularity conditions of Assumption~\ref{assum:regularity}.
\end{remark}
\section{Visualization of Region Refinement Mask Across Layers}
\label{sec:appendix_rrm_vis}

Figs.~\ref{fig:rrm_visualization1}, \ref{fig:rrm_visualization2}, and \ref{fig:rrm_visualization3} provide a comprehensive layer-by-layer visualization of the proposed Region Refinement Mask (RRM) mechanism applied to a representative input sample.
The visualization reveals two key observations that validate the effectiveness of RRM.

\paragraph{Effective Filtering of Outlier Tokens.}
The RRM consistently suppresses spurious, semantically irrelevant tokens across all layers.
As shown in the per-region maps, tokens that do not conform to the spatial and semantic structure of each region are progressively masked out, producing increasingly clean and compact region representations.
This confirms that RRM successfully identifies and removes outlier Tokens that would otherwise introduce noise into the feature aggregation process, thereby improving the overall discriminability of the learned representations.

\paragraph{Layer-wise Distribution of Outlier Tokens.}
A closer inspection of the layer-wise progression reveals a non-uniform distribution of outlier Tokens across network depth.
Specifically, a notably higher proportion of tokens are filtered in the \emph{shallow layers} (\textit{Layer~1}--\textit{Layer~4}), where low-level features have not yet been fully contextualized, and in the \emph{final layers} (\textit{Layer~21}--\textit{Layer~24}), where task-specific representations are being consolidated and certain tokens may drift away from the dominant semantic region.
In contrast, the intermediate layers (\textit{Layer~5}--\textit{Layer~20}) exhibit relatively stable region masks, suggesting that the mid-network features are more coherent and require less aggressive filtering. 
This layer-wise pattern is consistent with the COR evolution in Fig.~1(b) of the main text, where intermediate layers exhibit higher COR values, and provides empirical support for the choice of the intermediate fusion layer $l_{\text{mid}}$ in the Hierarchical Representation Integration.

\begin{figure*}[htbp]
    \centering
    \includegraphics[width=0.65\linewidth]{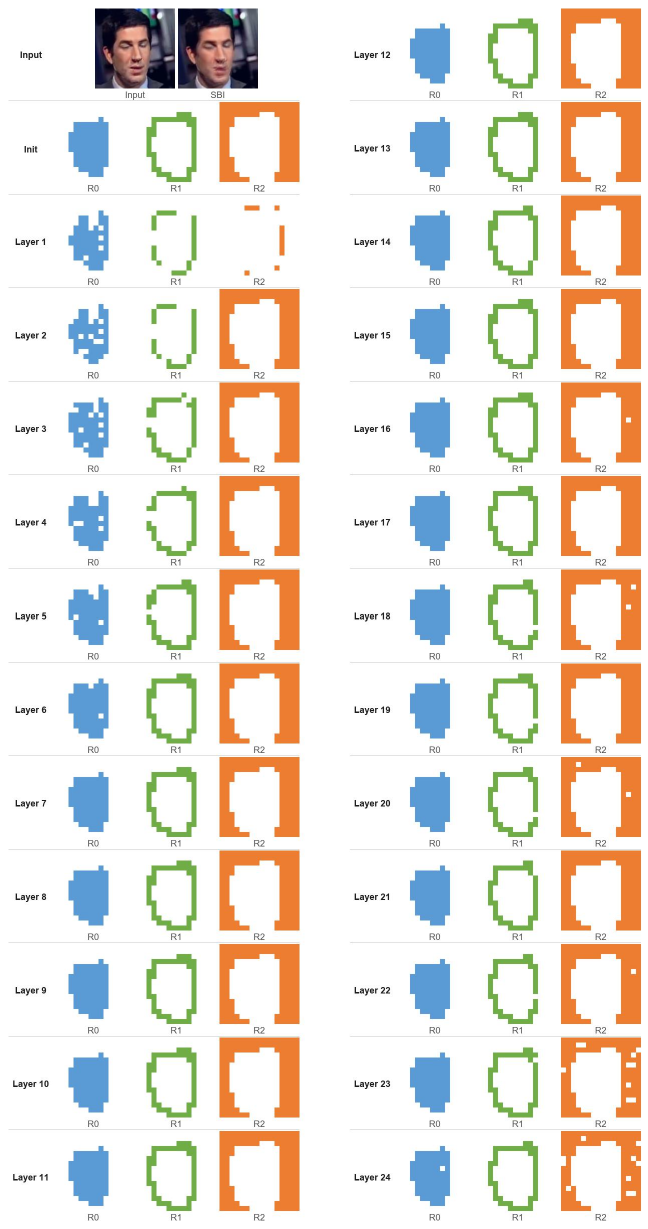}
    \caption{
        Layer-wise visualization of the Region Refinement Mask (RRM) for sample 1. Each row displays the predicted mask components (Region~0, Region~1, and Region~2) at a specific network layer, accompanied by the input image and the SBI baseline for reference. The progression is visualized from the initialization stage (\textit{Init}) through all 24 transformer layers (\textit{Layers 1 to 24}). Within each map, higher intensity areas indicate tokens retained after refinement, whereas darker regions correspond to suppressed outlier tokens.
    }
    \label{fig:rrm_visualization1}
\end{figure*}

\begin{figure*}[htbp]
    \centering
    \includegraphics[width=0.65\linewidth]{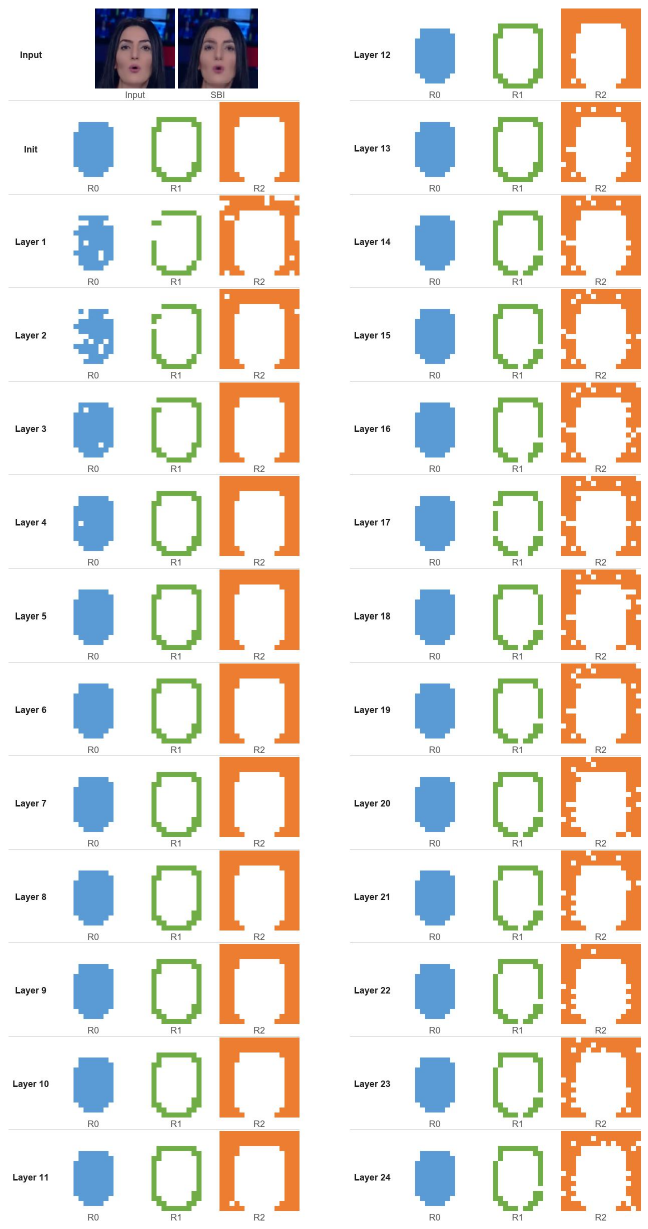}
    \caption{
    Layer-wise visualization of RRM for sample 2. 
    The layout and color coding follow Fig.~\ref{fig:rrm_visualization1}.
}
    \label{fig:rrm_visualization2}
\end{figure*}

\begin{figure*}[htbp]
    \centering
    \includegraphics[width=0.65\linewidth]{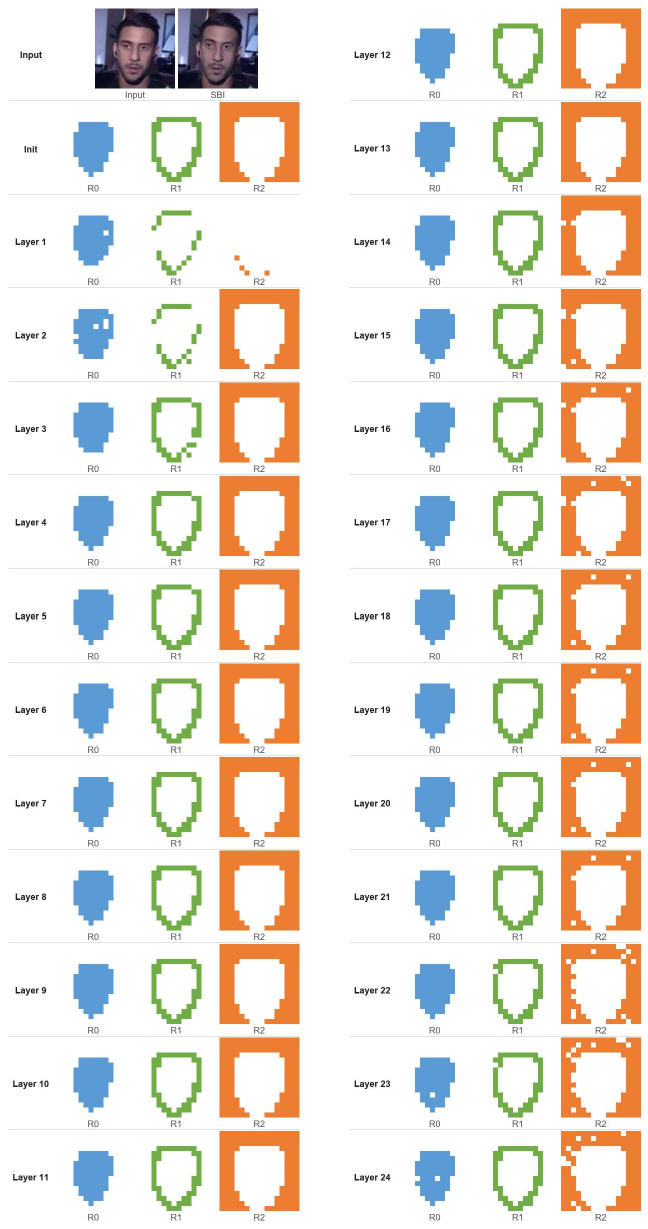}
    \caption{
    Layer-wise visualization of RRM for sample 2. 
    The layout and color coding follow Fig.~\ref{fig:rrm_visualization1}.
}
    \label{fig:rrm_visualization3}
\end{figure*}


\end{document}